\documentclass[10pt]{article}
\usepackage[accepted]{icml2025modified}
\usepackage[utf8]{inputenc} %
\usepackage[T1]{fontenc}    %
\usepackage{csquotes}
\usepackage{url}            %
\usepackage{booktabs}       %
\usepackage{amsfonts}       %
\usepackage{nicefrac}       %
\usepackage{microtype}      %
\usepackage{xcolor}         %
\usepackage{graphicx}
\usepackage{adjustbox}

\DeclareSymbolFont{bbold}{U}{bbold}{m}{n}
\DeclareSymbolFontAlphabet{\mathbbold}{bbold}
\usepackage{amsmath}
\usepackage{amsfonts}
\usepackage{amssymb}
\usepackage{amsbsy}
\usepackage{isomath}
\usepackage{varioref}
\usepackage{mathtools}

\usepackage{amsthm}
\usepackage{booktabs}

\usepackage{cleveref}
\crefname{assumption}{assumption}{assumptions}

\usepackage{thmtools,thm-restate}

\newtheorem{corollary}{Corollary}
\newtheorem{remark}{Remark}

\newtheorem{theorem}{Theorem}
\newtheorem{assumption}{Assumption}
\newtheorem{definition}{Definition}

\newcommand \E {\ensuremath{\mathbb{E}}}
\newcommand \He {\ensuremath{\mathbb{H}}}
\renewcommand \Pr {\ensuremath{\mathbb{P}}}
\newcommand \Q {\ensuremath{\mathbb{Q}}}
\newcommand \R {\mathop{\mbox{\ensuremath{\mathbb{R}}}}\nolimits}

\newcommand{\Otilde}{\widetilde{O}}
\newcommand{\Thetatilde}{\widetilde{\Theta}}
\newcommand{\Omegatilde}{\widetilde{\Omega}}

\newcommand{\subgc}{\alpha_{\theta^*}}
\newcommand{\subgM}{\alpha_{\M}}

\newcommand \defn {\mathrel{\triangleq}}

\newcommand \argmax{\mathop{\rm arg\,max}}

\DeclareMathAlphabet{\mathpzc}{OT1}{pzc}{m}{it}

\newcommand \prior {\Pr_0}

\newcommand \bel {\Pr}

\newcommand \lik {\mathcal{L}}
\newcommand \ball {\Xi}

\newcommand \pol {\pi}

\newcommand \data {\hist}
\newcommand \datat {\histl}
\newcommand \datatp {\histlp}
\newcommand \datatm {\histlm}
\newcommand{\hist}{\mathcal{H}}
\newcommand{\histl}{\mathcal{H}_\epindex}
\newcommand{\histlp}{\mathcal{H}_{\epindex+1}}
\newcommand{\histlm}{\mathcal{H}_{\epindex-1}}

\newcommand \BRT {\mathrm{BR}(\totalT)}

\newcommand {\lsirep}{\alpha_{\Bar{R},\epindex}}
\newcommand {\lsipep}{\alpha_{\Bar{\transition},\epindex}}
\newcommand {\lsiep}{\alpha_{\epindex}}

\newcommand {\lsiz}{\alpha_z}
\newcommand {\lsi}{\alpha}
\newcommand{\epindex}{l}
\newcommand{\Epindex}{\tau}
\newcommand{\totalT}{T}
\newcommand{\transition}{\mathcal{T}}

\newcommand{\epslange}{\epsilon_{\text{post}, \epindex}}
\newcommand{\epslangem}{\epsilon_{\text{post}, \epindex-1}}

\DeclarePairedDelimiterX{\infdivx}[2]{(}{)}{%
  #1\;\delimsize\|\;#2%
}

\newcommand{\KL}{\mathrm{KL}\infdivx}
\DeclarePairedDelimiter{\norm}{\lVert}{\rVert}

\newcommand{\M}{M}

\newcommand{\Me}{{M_\epindex}}

\newcommand{\pole}{\pi_\epindex}

\newcommand{\dparam}{D}

\newcommand \todoi[1] {}

\makeatletter
\let\save@mathaccent\mathaccent
\newcommand*\if@single[3]{%
  \setbox0\hbox{${\mathaccent"0362{#1}}^H$}%
  \setbox2\hbox{${\mathaccent"0362{\kern0pt#1}}^H$}%

  }
\newcommand*\rel@kern[1]{\kern#1\dimexpr\macc@kerna}
\newcommand*\widebar[1]{\@ifnextchar^{{\wide@bar{#1}{0}}}{\wide@bar{#1}{1}}}
\newcommand*\wide@bar[2]{\if@single{#1}{\wide@bar@{#1}{#2}{1}}{\wide@bar@{#1}{#2}{2}}}
\newcommand*\wide@bar@[3]{%
  \begingroup
  \def\mathaccent##1##2{%
    \let\mathaccent\save@mathaccent
    \if#32 \let\macc@nucleus\first@char \fi
    \setbox\z@\hbox{$\macc@style{\macc@nucleus}_{}$}%
    \setbox\tw@\hbox{$\macc@style{\macc@nucleus}{}_{}$}%
    \dimen@\wd\tw@
    \advance\dimen@-\wd\z@
    \divide\dimen@ 3
    \@tempdima\wd\tw@
    \advance\@tempdima-\scriptspace
    \divide\@tempdima 10
    \advance\dimen@-\@tempdima
    \ifdim\dimen@>\z@ \dimen@0pt\fi
    \rel@kern{0.6}\kern-\dimen@
    \if#31
      \overline{\rel@kern{-0.6}\kern\dimen@\macc@nucleus\rel@kern{0.4}\kern\dimen@}%
      \advance\dimen@0.4\dimexpr\macc@kerna
      \let\final@kern#2%
      \ifdim\dimen@<\z@ \let\final@kern1\fi
      \if\final@kern1 \kern-\dimen@\fi
    \else
      \overline{\rel@kern{-0.6}\kern\dimen@#1}%
    \fi
  }%
  \macc@depth\@ne
  \let\math@bgroup\@empty \let\math@egroup\macc@set@skewchar
  \mathsurround\z@ \frozen@everymath{\mathgroup\macc@group\relax}%
  \macc@set@skewchar\relax
  \let\mathaccentV\macc@nested@a
  \if#31
    \macc@nested@a\relax111{#1}%
  \else
    \def\gobble@till@marker##1\endmarker{}%
    \futurelet\first@char\gobble@till@marker#1\endmarker
    \ifcat\noexpand\first@char A\else
      \def\first@char{}%
    \fi
    \macc@nested@a\relax111{\first@char}%
  \fi
  \endgroup
}
\makeatother

\usepackage[labelformat=simple]{subcaption}

\usepackage{url}
\usepackage[space]{grffile}     %
\allowdisplaybreaks
\usepackage{wrapfig}

\usepackage[british]{babel} %

\usepackage{amsmath,amssymb,amscd,latexsym,dsfont}
\usepackage{textcomp}
\usepackage{pgfplots}
\pgfplotsset{compat=1.18}
\pgfmathdeclarefunction{gauss}{2}{%
  \pgfmathparse{1/(#2*sqrt(2*pi))*exp(-((x-#1)^2)/(2*#2^2))}%
}
\pgfmathdeclarefunction{gausssin}{2}{%
  \pgfmathparse{1/(#2*sqrt(2*pi))*exp(-((x-#1)^2)/(2*#2^2))+0.1*(e^(-(0.25*(x-#1-1)^2)))*(1+sin(400*x))}%
}
\pgfmathdeclarefunction{gaussmix}{2}{%
  \pgfmathparse{1/((#2*1.3)*sqrt(2*pi))*exp(-((x-#1)^2)/(2*(#2*1.3)^2)) + 1/(#2*sqrt(2*pi))*exp(-((x-3-#1)^2)/(2*#2^2))}%
}
\usepackage{psfrag}
\usepackage{rotating}

\usepackage{enumitem}
\usepackage{pdflscape}
\usepackage[normalem]{ulem}

\usepackage[final]{pdfpages}
\usepackage{xspace}
\usepackage{relsize}

\usepackage{microtype}

\usepackage[disable]{todonotes}

\icmltitlerunning{Isoperimetry is All We Need}

\begin{document}

\onecolumn
\icmltitle{Isoperimetry is All We Need:  \\ Langevin Posterior Sampling for RL with Sublinear Regret}

\icmlsetsymbol{equal}{*}

\begin{icmlauthorlist}
\icmlauthor{Emilio Jorge}{chs,gu}
\icmlauthor{Christos Dimitrakakis}{sch,os}
\icmlauthor{Debabrota Basu}{inr}
\end{icmlauthorlist}

\icmlaffiliation{chs}{Chalmers University of Technology, Sweden}
\icmlaffiliation{gu}{University of Gothenburg, Sweden}
\icmlaffiliation{sch}{University of Neuch\^{a}tel, Switzerland}
\icmlaffiliation{os}{University of Oslo, Norway}
\icmlaffiliation{inr}{Univ. Lille, Inria, CNRS, Centrale Lille, UMR 9189 – CRIStAL}
\icmlcorrespondingauthor{Emilio Jorge}{emilio@jorge.se}

\icmlkeywords{Machine Learning, ICML, isoperimetry, reinforcement learning, log-sobolev inequality, Bayesian regret}

\vskip 0.3in
\printAffiliationsAndNotice{} %

\begin{abstract}
Common assumptions, like linear or RKHS models, and Gaussian or log-concave posteriors over the models, do not explain practical success of RL across a wider range of distributions and models. 
    Thus, we study how to design RL algorithms with sublinear regret for isoperimetric distributions, specifically the ones satisfying the Log-Sobolev Inequality (LSI). 
    LSI distributions include the standard setups of RL theory, and others, such as many non-log-concave and perturbed distributions. 
    First, we show that the Posterior Sampling-based RL (PSRL) algorithm yields sublinear regret if the data distributions satisfy LSI and some mild additional assumptions.
    Also, when we cannot compute or sample from an exact posterior, we propose a Langevin sampling-based algorithm design: LaPSRL. 
    We show that LaPSRL achieves order-optimal regret and subquadratic complexity per episode. 
    Finally, we deploy LaPSRL with a Langevin sampler-- SARAH-LD, and test it for different bandit and MDP environments. 
    Experimental results validate the generality of LaPSRL across environments and its competitive performance with respect to the baselines.      

\end{abstract}
\section{Introduction}
     The last decade has seen significant advances in Reinforcement Learning (RL), both in terms of theoretical understanding and  practical applications. However, the theory does not always apply to real-world settings-- exposing a theory-to-practice gap. For complex environments, RL algorithms often use a probabilistic approximation of the environment. In order to analyse them theoretically, we often assume linear~\citep{MAL-042}, bilinear~\citep{ouhamma2022bilinear}, or reproducible kernel~\citep{chowdhury2019online} type parametric models, and Gaussian or log-concave posteriors for Bayesian RL  algorithms~\citep{chowdhury2019online, osband2017posterior}. In this paper, \textit{we aim to narrow this theory-to-practice gap by studying whether we can achieve the desired, i.e. sublinear, regret guarantees for the larger class of isoperimetric distributions}.

\begin{figure}[h!]
\centering
\begin{adjustbox}{width=0.4\linewidth}\centering
\begin{tikzpicture}
\begin{axis}[
  no markers, domain=0:10, samples=300,
  axis lines*=left, xlabel=$\theta$, ylabel=$\bel(\theta)$,
  every axis y label/.style={at=(current axis.above origin),anchor=south},
  every axis x label/.style={at=(current axis.right of origin),anchor=west},
  xlabel style={font=\Large},
  ylabel style={font=\Large},
  height=5cm, width=12cm,
  xtick=\empty, ytick=\empty,
  enlargelimits=false, clip=false, axis on top,
  grid = major
  ]
  \addplot [very thick,black] {gaussmix(4,1)};
\end{axis}
\end{tikzpicture}
\end{adjustbox}\hspace{2em}
\begin{adjustbox}{width=0.4\linewidth}
\begin{tikzpicture}
\begin{axis}[
  no markers, domain=0:10, samples=300,
  axis lines*=left, xlabel=$\theta$, ylabel=$\bel(\theta)$,
  every axis y label/.style={at=(current axis.above origin),anchor=south},
  every axis x label/.style={at=(current axis.right of origin),anchor=west},
  xlabel style={font=\Large}, %
  ylabel style={font=\Large},
  height=5cm, width=12cm,
  xtick=\empty, ytick=\empty,
  enlargelimits=false, clip=false, axis on top,
  grid = major
  ]
  \addplot [very thick,black] {gausssin(4,1)};
\end{axis}
\end{tikzpicture}
\end{adjustbox}%
\caption{Examples of log-Sobolev distributions.}\label{fig:example_sobolev}%
\end{figure}
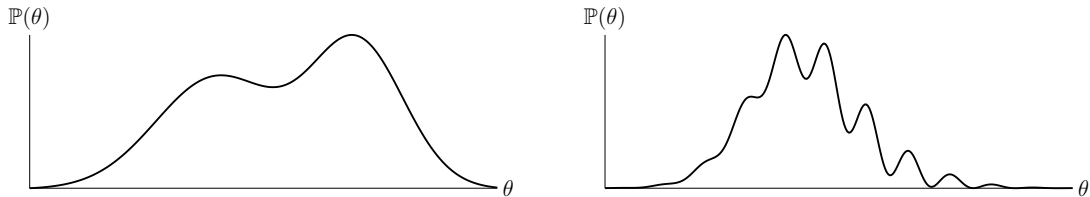

     \textbf{Isoperimetry} relates to the ratio between the area of the boundary and the volume of a set. It is known that some isoperimetric condition is needed for rapid mixing of Markov chains to avoid the risk of getting stuck in bad regions~\citep{stroock1992equivalence}, and thus, motivates us to study isoperimetric distributions in RL. In addition, isoperimetric distributions include all the aforementioned setups studied in RL theory, and in addition,  many non-log-concave and perturbed versions of log-concave distributions (Figure~\ref{fig:example_sobolev}) as well as mean field neural networks~\citep{pmlr-v151-nitanda22a}. In fact, we will see that any posterior with a bounded likelihood function and a log-Sobolev prior will be log-Sobolev, which would include complex setups such as some forms of Bayesian neural networks.  In the optimization and sampling literature, isoperimetry is used as a minimal condition to conduct efficient and controlled sampling from target distribution(s)~\citep{NEURIPS2019_65a99bb7}, while ensuring proper concentration of empirical statistics~\citep{ledoux2006concentration}. Among the different forms of isoperimetric inequalities (e.g. Poincar\'e, modified log Sobolev etc.), \textbf{we consider the Log Sobolev Inequality (LSI)}~\citep{bakry2014analysis} in this paper.

\begin{table*}[t!]
    \centering
    \caption{Summary of assumptions and results of Exact and Approximate PSRL type algorithms with Bayesian regret.}\label{tab:relatedregret}
    \resizebox{\textwidth}{!}{
    \begin{tabular}{l l l l}
    \toprule
        Algorithm &  $\BRT$ & Assumption &Total gradient complexity \\
        \midrule
        PSRL \cite{osband2013more} & $\Otilde(HS\sqrt{AT})$& Tabular, exact posterior &  - \\
        \cite{moradipari2023improved} & $\Otilde(H^{1.5}\sqrt{SAT})$ & Tabular, exact posterior & -  \\        
        \cite{pmlr-v139-fan21b} &$\Otilde(d H^{1.5}\sqrt{T})$ & Linear MDP, exact posterior & - \\
        \cite{chowdhury2019online} & $\Otilde(\sqrt{d HT})$ & Kernel MDP, exact posterior & - \\
        \bf{PSRL} (\Cref{thm:subgregret}) & $\Otilde(H \totalT^{0.75})$ &LSI $\lik$ with constant $\lsi$, exact posterior & -\\
        \bf{PSRL} (\Cref{thm:subgregret_lsi}) & $\Otilde(\sqrt{d H \totalT})$ &LSI $\bel(M)$, linear growth on $\lsi$, exact posterior& -\\
        \midrule
        \cite{pmlr-v162-xu22p} & $\Otilde(d^{1.5}\sqrt{T})$ & Lin. Bandits with cond. number $\kappa$, subG rewards& $\Otilde(\kappa T^{2})$ \\
        \cite{kuang2023posterior} & $\Otilde(d^{1.5}H^{1.5}\sqrt{T})$& Linear MDP, episodic Delay&$\Otilde(T^2)$ \\
        \cite{haque2024more} &$\Otilde(H^{1.5}d\sqrt{T}) $&Linear MDP& $\Otilde(\totalT^2/\sqrt{d}) $\\
        \cite{pmlr-v202-karbasi23a} &$ \Otilde(d\mathfrak{s}\sqrt{T})$&Inf. horizon with span $\mathfrak{s}$, $d\ll |\mathcal{S}||\mathcal{A}|$, strongly log-concave &  $\Otilde(1)$ (due to log-conc.)\\
        \midrule
       \bf{LaPSRL} (\Cref{lemma:regret_lin_growth}) & $\Otilde(\sqrt{d H \totalT })$ &LSI $\bel(M)$, linear growth on $\lsi$& $\Otilde(T\tau + \totalT^{1.5}\tau/d)$\\
       \bf{LaPSRL}  (\Cref{corr:sample_complexity})& $\Otilde(\sqrt{T}g(\cdot))$ & LSI $\bel(M)$, policy with $ \BRT=\Otilde(\sqrt{T}g(\cdot))$ for exact post. &$\Otilde\left(\sum_{\epindex=1}^\Epindex \frac{H^3 \epindex^3 }{\lsiep^2} + \frac{d H^{4.5}\epindex^{3.5}}{\lsiep^2 g(\cdot)^2 }\right) $ \\
        \bottomrule
    \end{tabular}}%
\end{table*}
    
\textbf{Posterior Sampling-based RL (PSRL).} We focus on \emph{Bayesian} RL, and in particular PSRL algorithms~\citep{russo2020tutorial, osband2013more}.   PSRL, also called Thompson sampling~\citep{thompson1933lou}, is a a Bayesian algorithm requires a prior distribution to be defined over models and a posterior to be calculated as more data is collected. Periodically, a model is sampled from the posterior distribution, which is then used to create a policy. The policy is then used to act in the environment.  PSRL has been successful both theoretically and experimentally, but its exact inference and sampling are only possible in very simple settings. On the other hand, naive approximations lead to linear regret. A recent line of research has examined how to obtain sublinear regret with approximate PSRL. However, this has so far been limited to either bandit problems or required strong assumptions. Our work is \emph{more general}, both because it applies to general RL problems, and because the assumptions strictly generalise those of previous works.

\textbf{Langevin Sampling-based PSRLs.} A growing approach in this direction is to use Langevin-based approximate sampling methods, %
which are known to be generic and efficient in optimisation, sampling, and deep learning literature. \citet{mazumdar2020approximate, pmlr-v238-zheng24b, pmlr-v206-huix23a} propose Langevin-based PSRL algorithms for multi-armed bandits that achieve order-optimal regret.   Similarly, \citet{pmlr-v162-xu22p} extend these ideas to linear contextual bandits and 
\citet{ishfaq2023provable, haque2024more,pmlr-v202-karbasi23a} bring Langevin-based PSRL to Markov Decision Processes (MDPs). A crucial limitation of all these works are that the regret guarantees are valid \textit{only for log-concave or linear problems}, which we relax as shown in Table~\ref{tab:relatedregret}. Note that constant gradient complexity in \cite{pmlr-v202-karbasi23a} is possible due to strong log-concavity and exponential episode length, and is non-trivial to generalise. 

However, the sampling literature has shown that Langevin methods are efficient for distribution fulfilling isoperimetric conditions, such as the ones satisfying LSI~\citep{svrgld,NEURIPS2019_65a99bb7,kinoshita2022improved}. This motivates us to propose a generic algorithm that can work for any distribution satisfying LSI, both for bandits and MDPs, and also to study the minimum conditions required to achieve sublinear regret.
Specifically, we ask:

1. \textit{Is isoperimetry of posteriors enough to ensure efficient execution of PSRL-type algorithms?}\\
2. \textit{Can we use Langevin sampling-based algorithms to approximate the isoperimetric posteriors and still obtain an efficient approximate PSRL algorithm?}

\textbf{Our contributions} address these questions affirmatively and more. Specifically, we 

1. Prove that \textit{PSRL can achieve sublinear regret for posteriors satisfying LSI} under some mild conditions if we can compute and sample from the exact posteriors. This result broadens the scenarios where PSRL is proven to be efficient to a new and wider family of posteriors. We show this for two cases: both when the data likelihood satisfies LSI and when the posterior distributions over MDPs satisfies LSI.

2. Propose a generic PSRL-algorithm, called \textbf{LaPSRL}, that \textit{uses a Langevin-based sampling to compute approximate posterior distributions}. A generic regret analysis of LaPSRL shows it \textit{can achieve $\tilde{\mathcal{O}}(\sqrt{T})$ regret if the approximate sampling algorithms allow the posterior to contract linearly}, where $T$ is the number of interactions. Then, we show that if we deploy LaPSRL with SARAH-LD~\citep{kinoshita2022improved}, which is an efficient Langevin sampling algorithm, we only need a polynomial number of samples with respect to the MDP parameters (both with and without chaining the samples). Conducting our regret analysis requires generalising the classical analysis of PSRL for LSI and studying the contraction of posteriors under Langevin dynamics.

3. Show that \textit{LaPSRL deployed with SARAH-LD achieves sublinear regret across different environments}, including Gaussian, Mixtures of LSI distributions as well as any log-concave distribution or mixtures of them. We show that LaPSRL's regret is competitive w.r.t. the existing PSRL algorithms with approximate posteriors. %

4. \textit{Experimentally demonstrate that LaPSRL with SARAH-LD yields sublinear regret} for bandits with Gaussians and mixture of Gaussians as posteriors, as well as continuous MDP experiments with Linear Quadratic Regulators (LQRs) and neural networks with Gaussian priors. Numerical results validate that LaPSRL performs competitively with the baselines across environments and posterior distributions.

\section{{Problem Setup \& Background}}
Before proceeding to the contributions, we first formally state the problem of episodic RL. Then we summarise PSRL for episodic RL and Langevin based sampling techniques, which are the main pillars of our work.

\textbf{Notations.} Complexity notations $\Otilde, \Omegatilde, \Thetatilde$ ignore sub- and poly-logarithmic terms.
\Cref{tab:Notation} summarises the notations.

\textbf{Problem Formulation: Episodic Reinforcement Learning (RL).} We consider episodic finite-horizon MDPs (aka \textit{Episodic RL})~\citep{osband2013more, azar2017minimax}. An MDP $\M$ is a tuple $\langle \mathcal{S}, \mathcal{A}, \transition, R, H \rangle$ with states $s\in \mathcal{S} \subseteq \R^d$, and actions $a\in \mathcal{A}$. %
In every episode, the agent interacts with the environment for $H$ steps: The episode $\epindex$ starts at some state $s_{\epindex,1}$. Then, for $t\in[H]$, the agent draws action $a_{\epindex,t}$ from a policy $\pi_t(s_{\epindex,t})$, observes the reward $R(s_{\epindex,t},a_{\epindex,t}) \in \R$, and transits to a state $s_{t+1}^\epindex \sim \mathcal{T}(.\mid s_{\epindex,t},a_{\epindex,t})$. This interaction is done for a total of $\Epindex$ episodes, which is commonly unknown a priori.  We denote the total interactions with the environment as $\totalT\triangleq \Epindex H$. When there is only one state, or the state does not depend on the action, this problem reduces to multi-armed bandits~\citep{lattimore2020bandit}.

The performance of a policy $\pi$ is measured by the expected total reward $V_{1}^\pi$ w.r.t. an initial state $s$. We define the value function at $h\in [H]$ as
$V_M^{\pi, h} (s)  \triangleq \E_{\M,\pi}\left[\sum_{t=h}^H R(s_t,a_t) \mid s_h =s \right]$. The policy maximising these value functions is the optimal policy $\pi^*$.

\textbf{Background: Bayesian RL.}
In the Bayesian setting, we first sample an MDP $\M$ from a prior distribution $\prior(\M)$. We then compute a policy for it and collect the data (aka history) by playing the policy, i.e. $\datat \triangleq \{s_{1,1},a_{1,1}, \ldots, s_{\epindex-1,H},a_{\epindex-1,H} \}$.
We then construct a posterior distribution $\bel(\M\mid \datat)$ on $\M$ using $\datat$. 

  Regret quantifies how much worse the learning policy is than an oracle policy. In Bayesian RL it is standard to consider the \textit{Bayesian regret}~\citep{osband2013more, doi:10.1287/moor.2014.0650,NEURIPS2021_edb446b6}, i.e. the expectated regret over the possible MDPs and trajectories.  Given a prior $\prior(\M)$ over MDPs, Bayesian regret of a policy $\pi$ is 
  \begin{equation}\label{eq:bayes_regret}
        \BRT \triangleq \E_{\prior,\pi}\left[ \sum_{\epindex=1}^\Epindex V_{\pi^*,1}^{\M}(s_{\epindex,1})-V_{\pi_\epindex,1}^{\M}(s_{\epindex,1})\right] \,.
  \end{equation}
 A sub-linear Bayesian regret means that the RL algorithm can solve all MDPs under prior, except for a set with measure zero~\citep{osband2013more, doi:10.1287/moor.2014.0650}. To be specific, \citet{NEURIPS2021_edb446b6} states that sublinear Bayesian regret is preserved under a misspecified prior as long as the true prior is absolutely continuous w.r.t. the misspecified prior. 
Additionally, work has been done on setting these priors robustly as in \citep{pmlr-v206-buening23a}.
Our results do not just hold for a specifically chosen prior, but for \emph{any prior} satisfying our assumptions.

 A popular and successful Bayesian RL approach is to sample an MDP $\Me\sim \bel(\M\mid \datat)$  and play the optimal policy for $\Me$ for one episode before updating the posterior and resampling. This algorithm is known as \textbf{PSRL} \citep{osband2013more}, or as Thompson sampling in bandits~\citep{thompson1933lou}. We illustrate a pseudocode of PSRL in \Cref{alg:psrl}.
 
   \begin{algorithm}[t!]
    \caption{PSRL}\label{alg:psrl}
    \begin{algorithmic}[1]
    \STATE \textbf{Input}: Likelihood $\lik(x|\M)$, Prior $\prior(\M)$
    \FOR{$\epindex=1,2,\ldots$}
    \STATE Sample $\Me \sim \bel(\M \mid \datat)$
    \STATE Play $\pi^*(\Me)$ till horizon $H$ to obtain $\{x_i\}_{i=H(\epindex-1)}^{H\epindex}$
    \STATE $\data_{\epindex+1} \gets \datat \cup \{x_i\}_{i=H(\epindex-1)+1}^{H\epindex}$
    \ENDFOR
    \end{algorithmic}
    \end{algorithm}
    
\textbf{Background: Sampling with Langevin dynamics.} Now, we discuss the other fundamental component of our work, i.e. Langevin sampling.
This involves sampling from a target distribution  $d\nu \propto e^{- \gamma F(\theta)}d\theta$ over a set of parameters $\theta \in \R^d$, with $F: \R^d \to \R$ defined as an $n$-sample average loss $\frac{1}{n} \sum_{i=1}^n f_i(\theta)$ and $\gamma=n$, where $f_i$ is the loss for sample $x_i$.
In case of Bayesian posteriors, we define $f_i(\theta) \triangleq - \frac{\log \prior(\theta)}{n} - \log \lik(x_i|\theta)$, i.e. each $f_i$ corresponds to the log-likelihood of data point $x_i$ and its share of the log prior. 

In continuous-time, Langevin methods can sample exactly from a posterior~\citep{NEURIPS2019_65a99bb7}. In practice, discretisation makes this impossible, but using a Langevin gradient descent algorithm allows for sampling from the target distribution with a controlled bias, under conditions of isoperimetry~\citep{kinoshita2022improved}. For this to work, we also need a property on smoothness:
\begin{assumption}[L-smoothness]\label{ass:smooth} 
If $f_i$ is twice differentiable for all $i\in [n]$, and $\forall x \in \mathbb{R}^d,\left\|\nabla^2 f_i(x)\right\| \leq L$, i.e. $f_i$ is L-smooth. Additionally, this implies that $F$ is also $L$-smooth. 
\end{assumption}

We now start with an introduction to log-Sobolev distributions. There is a rich literature on log-Sobolev distributions---a summary of which can be found in~\citep{LSIessentials,NEURIPS2019_65a99bb7}.  

\begin{definition}[log-Sobolev inequality]\label{ass:lsi}
	A distribution $\nu$ satisfies the log-Sobolev inequality (LSI) with a constant $\lsi > 0$ if,  for all smooth functions $g:\R^d\to \R$ with $\E_\nu[g^2]\le \infty$, 
        \begin{align}\tag{LSI}\label{eqn:LSI}
           \hspace*{-.8em} \E_\nu[g^2\log g^2] - \E_\nu[g^2]\log\E_\nu[g^2]\le \frac{2}{\lsi}\E_\nu[||\nabla g||^2].
        \end{align}
	\end{definition}
    
    Obtaining the LSI constant\footnote{Note that some works use an inverse definition of LSI constant, i.e. $\lsi'=\frac{1}{2\lsi}$, leading to some confusion. We stick to Definition~\ref{ass:lsi}.} $\alpha$ is not always trivial.  We can determine whether a distribution is LSI through Lyapunov conditions, integral conditions, local inequalities and tools from optimal transport as well as mixture decomposition~\citep{cattiaux2010note,wang2001logarithmic,barthe2008mass,CHEN2021109236,koehler2023statistical}. The most common tool is the Bakry-\'{E}mery criterion. 
    \begin{theorem}[Bakry-\'{E}mery criterion]
    \label{theorem:bakry-emery}
If for distribution $\nu$,
$- \nabla_{\theta}^2 \log \nu  \geq \lsi I_d $,
 where the inequality is the Loewner order, $I_d$ the identity matrix of dimension $d$ and $\theta$ the parametrization of $\nu$,  then $\nu$ fulfils LSI  with constant $\lsi$.
   \end{theorem}

\textit{Remark: \underline{Log-concavity vs. LSI.}}
    \Cref{theorem:bakry-emery} shows that log-concave distributions imply LSI. A distribution $\nu(\theta)$ is log-concave if $\log\nu(\theta)$ is concave in $\theta$. Log-concavity is a commonly used condition in sampling and RL~\citep{mazumdar2020approximate,NIPS2017_3621f145,pmlr-v80-abeille18a} but this is significantly more restrictive than log-Sobolev. For example, log-concave distributions cannot be multimodal. Examples of LSI distributions given in \Cref{fig:example_sobolev} show that multiple modes are generally not an issue for LSI. 

    An \emph{equivalent} way of expressing \eqref{eqn:LSI} is $\KL{\rho}{\nu} \le \frac{1}{2\lsi} J_\rho$, where $\rho\triangleq\frac{g^2\nu}{\E_\nu[g^2]}$ and  $J_{\rho}\triangleq\mathbb{E}_\rho\left[\left\|\nabla \log \frac{\rho}{\nu}\right\|^2\right]$ is the relative Fisher information of $\rho$ with respect to $\nu$. %

   Another operation on log-Sobolev distributions that preserves the LSI property while breaking log-concavity is a bounded perturbation. The result is due to~\cite{holley1986logarithmic} but is presented here as per~\cite{steinergao2021feynmankac}.
     \begin{theorem}[\cite{steinergao2021feynmankac}]\label{theorem:holley}
         Assume that $d\mu \propto e^\Phi d\nu$, where $\nu$ is a probability measure that satisfies LSI and $\Phi$ is continuous and bounded. Then $\mu$ satisfies a LSI with ${\alpha_\nu} \le e^{2(\sup(\Phi)-\inf(\Phi))}{\alpha_\mu}.$
     \end{theorem}\vspace*{-.5em}
     
     In some cases, unbounded perturbations to LSI distributions can still be LSI~\citep{steinergao2021feynmankac}.  LSI is also preserved under a Lipschitz-transformation~\citep{NEURIPS2019_65a99bb7}, while log-concavity is not. Also, if the distribution is factorizable in log-Sobolev components, then the product is log-Sobolev with an LSI constant that is minimum among the factored components~\citep{ledoux2006concentration}. Mixtures of log-Sobolev distributions are also log-Sobolev under conditions on the distance between the distributions (See \Cref{theorem:mixturelsi}). 

   \textit{Sub-Gaussian Concentration from LSI.} An important consequence of a distribution $\nu$ satisfying LSI with constant $\lsi_\nu$ is obtaining the Gaussian concentration of any Lipschitz function around its mean~\citep{bizeul2023logsobolev}. Specifically, for any function $g:\R^d \to \R$ with Lipschitz constant $L_g$, 
    \begin{align}\label{Eq:conc}
        \Pr_\nu(|g-\E_\nu[g]|\ge t)\le 2\exp\left(-\frac{\lsi_\nu t^2}{L_g^2}\right)\,.
    \end{align}\vspace*{-.5em}
    
    Under a curvature dimension condition, the reverse is also true, Gaussian concentration implies that the distribution is log-Sobolev, \citep[Theorem 8.7.2]{bakry2014analysis}.

\textbf{Background: SARAH-LD~\citep{kinoshita2022improved}.} Multiple algorithms are developed to perform biased Langevin sampling on log-Sobolev distributions~\citep{NEURIPS2019_65a99bb7,kinoshita2022improved}. In this paper, we focus on SARAH-LD (\Cref{alg:sarahld}), a variance-reduced version of Langevin dynamics that is state-of-the-art in terms of the KL divergence concentration between the sampled and the target distributions, i.e. $\KL{\hat{\nu}}{\nu}$. SARAH-LD allows us to control the bias, and trade off the computational complexity w.r.t. $\KL{\hat{\nu}}{\nu}$. Given $n$ previous samples, SARAH-LD needs  $\Otilde\left(\left(n+\frac{d \sqrt{n}}{\KL{\hat{\nu}}{\nu}}\right) \frac{1}{\lsi^2}\right)$ steps of stochastic gradient evaluations a new sample respresentative of the approximate posterior (also known as \textit{gradient complexity}). The complete result is in \Cref{thm:sarahldcomplex}.

\section{{PSRL for Exact posteriors}}
We first consider the convergence of posterior sampling (PSRL, Algorithm~\ref{alg:psrl}), when we have access to exact posterior distributions at each step.
We ask: \textit{can PSRL achieve sublinear regret if posterior or likelihood are isoperimetric?}

\textbf{Bounding Regret of PSRL.} First, we observe that following the series of works by~\cite{osband2017posterior,chowdhury2019online,pmlr-v130-chowdhury21b}, a generic three-step framework to bound the Bayesian regret ($\BRT$) of PSRL can be developed. 

\textit{Step 1:} At the first step of episode $l$, the total number of completed steps is $n=(l-1)H$. \citet{osband2013more} observes that for any $\sigma(\datat)$ measurable function $f$, which includes the value function, we have $\E[f(\Me)]=\E[f(\M)]$. Thus, we get from Equation~\eqref{eq:bayes_regret} that $\BRT = \E\left[ \sum_{\epindex=1}^\Epindex V_{\pi_l,1}^{\Me}(s_{\epindex,1})-V_{\pi_\epindex,1}^{\M}(s_{\epindex,1})\right]\,.$

\textit{Step 2:} \citet{chowdhury2019online} further shows that by a recursive application of the Bellman equation, we can decompose this regret into the expectation of a martingale difference sequence, and the difference between the next step value functions of the sampled and true MDPs. 
Specifically, 
$    \BRT= \E\left[ \sum_{\epindex,h=1,1}^{\Epindex, H} \mathcal{T}^{\pi_l}_{\Me,h}(V_{\pi_l,h+1}^{\Me})(s_{\epindex,h})-\mathcal{T}^{\pi_l}_{\M,h}(V_{\pi_\epindex,h+1}^{\M})(s_{\epindex,h})\right]$
and $\mathcal{T}^{\pi}_{\M,h}(V_{\pi,h+1}^{\M})(s_{\epindex,h}) =R(s,\pi(s,h)) +\E_{s,\pi(s,h)}[V|\M]$ is the Bellman operator at step $h$ of the episode under policy $\pi$ and MDP $\M$.

\textit{Step 3:} Finally, in the spirit of~\citep{pmlr-v130-chowdhury21b}, we use the transportation inequalities~\citep{boucheron2003concentration} to yield an upper bound of $\BRT$ as%
\begin{align*}
 \BRT &\leq H\sigma_R \E\Bigg[ \sum_{\epindex, h=1,1}^{\Epindex,H}\sqrt{2 \mathrm{KL}_{s_{l,h},a_{l,h}}(\M||\Me)}\Bigg]\\&\leq B_R+ 2H\sigma_R \sqrt{T \xi(T)}.    
\end{align*}

Here, $B_R$ and $\sigma_R^2$ bound the mean and variance of rewards for any $\M$, and $\mathrm{KL}_{s_{l,h},a_{l,h}}(\M||\Me)$ denotes the KL divergence between $\mathcal{T}_{\M}(\cdot|s_{l,h},a_{l,h})$ and $\mathcal{T}_{\M}(\cdot|s_{l,h},a_{l,h})$. 
\textit{The last inequality holds if we can show} that $\mathrm{KL}_{s_{l,h},a_{l,h}}(\M,\Me)$ is upper bounded by a constant or monotonically increasing polylogarithmic function, say $\xi(\epindex h)$, with probability at least $1-\frac{1}{\epindex h}$ for any $\epindex h>1$.
Thus, \textit{our strategy for proving sublinear regret for PSRL is via such a concentration bound.} 

\textbf{A Generic Result: Concentration of Bayesian Posterior of LSI Distributions.} Now, we prove an interesting and generic result that if the data distribution is isoperimetric and the prior is designed to have enough mass around the true MDP, we can achieve a polylogarithmic $\mathrm{KL}$-divergence concentration rate under the posterior distribution. In this setting we have $n$ i.i.d. observations $x^n = (x_1, \ldots, x_n)$ sampled from a distribution with parameter $\theta^*$: $x_i \sim \lik(x | \theta^*)$. We assume that the prior $\prior$ has a non-neglible mass around $\theta^*$:
\begin{assumption}[True-mass Prior] \label{ass:mass}
 There exists a measure $\omega$ such that $\prior$ has non-zero mass in all closed, compact sets $\ball \subset \Theta$ around $\theta^*$:
     $\int_{\theta\in\ball}  e^{- L_{\theta}\|{\theta^*}-{ \theta}\|}\prior(\theta) d\theta \geq \omega(\ball).$
 \end{assumption}

This condition allows deriving the concentration result.
   \begin{restatable}{theorem}{subgconf}\label{lemma:subgconf}
        If the data likelihood $\lik(x \mid \theta^*)$ is $\subgc$-LSI, \Cref{ass:mass} holds, %
        and  the log-likelihood function $\log \lik(x \mid \theta)$ is $L_x$ and $L_\theta$ Lipschitz in $x$ and $\theta$, respectively. Then, for any $n>0$ and $\theta \sim \bel(\theta|x^n)$, we obtain with probability at least $1-2\delta$ (for $\delta \in (0,1/2]$)
       \begin{align}
           \mathrm{KL}(\lik(x|\theta^*)||\lik(x|\theta))%
            \leq  \sqrt{\frac{4L_x^2 \ln\left(\frac{2}{\delta}\right)}{n\subgc}} +\frac{\ln\left(\frac{1}{\delta}\right)}{n}\,.
        \end{align}
    \end{restatable}
The detailed proof is in \Cref{proof:subgconf}.
This result shows a generic concentration of the posteriors for data likelihoods obeying LSI, and thus demands wider interest for any Bayesian learning framework. In RL, the data $x$ is the state-action pairs $(s_t,a_t)$ and $\theta$ represents the MDP $M$ while the data is generated by a policy $\pi$ from $M$.
   \begin{remark}[Prior Design]
Designing priors that have a small ball probability around the true parameter is common in Bayesian learning. For example, \cite{castillo2015bayesian} proposes such a prior for efficient Bayesian learning, and \cite{chakraborty2023thompson} carefully designed prior to prove near-optimal learning in high-dimensional bandits.
\end{remark}
\begin{remark}[Lipschitz Log-likelihood]
    To prove this result, we need an additional assumption of Lipschitzness of the log-likelihood function with respect to the data and the parameter. This holds true for most parametric distributions that does not consist of a Dirac distribution. For example, general exponential family distributions satisfy Lipschitzness for their natural parameters and sufficient statistics over data~\citep{efron2022exponential}. Bilinear exponential families~\citep{ouhamma2022bilinear} and any location-scale family~\citep{ferguson1962location} of distributions also do the same for their natural parameters and sufficient statistics.
\end{remark}

\textbf{Sublinear Regret of PSRL.} Now, we apply Theorem~\ref{lemma:subgconf} to PSRL and bound its Bayesian regret.
\begin{restatable}{lemma}{subgregret}\label{thm:subgregret}
    Under the conditions of \Cref{lemma:subgconf} with $M=\theta^*$ and the mean reward for the MDPs satisfying $|\Bar{R}_\M(s)|\le B_R \forall s$, Bayesian regret of PSRL satisfies
    \begin{align*}
    \BRT &= \tilde{\mathcal{O}}\left(B_R\left(1+ H \sqrt{T} + H \left(\frac{L_x^2}{\subgM}\right)^{1/4}\totalT^{3/4}\right)\right).
\end{align*}
\end{restatable}
The discussion earlier in this section are fleshed out in a detailed proof in \Cref{subgregret_proof}.
    The third term of $\mathcal{O}(\totalT^{3/4})$ in the Bayesian regret would be reduced to $\sqrt{T}$ if the posterior $\bel(\M \mid \datat)$ satisfies LSI with linearly increasing constants. We formalise this observation now. 
    
\textbf{Near-optimal Regret: Linear LSI Constant of Posteriors.} From~\citep{chowdhury2019online}, we observe that if we assume the next step value functions $\mathcal{T}^{\pi}_{\M,h}(V_{\pi,h+1}^{\M})(s_{\epindex,h})$ are mean-Lipschitz with respect to the state distributions, we can obtain an alternative approach to prove sublinear regret of PSRL. This results holds if additionally the rewards are bounded and Lipschitz, and the transitions are Lipschitz. Under this condition, we get the following results.
\begin{restatable}{theorem}{subgregretlsi}
\label{thm:subgregret_lsi}
 If the posterior distributions for mean rewards and transitions separately satisfy  LSI with constants $\{\lsirep\}$ and $\{\lsipep\}$, the mean reward for any MDP $\M$ is bounded: $|\Bar{R}_\M(s)|\le B_R \forall s$, the one step value function is Lipschitz in the state with parameter $L_\M$ as given in \Cref{ass:lip}, and the mean reward and mean transitions are $L_{\Bar{R}}$ and $L_{\Bar{\transition}}$ Lipschitz in $\M$, Bayesian regret of PSRL is bounded by 
 \begin{align*}
\BRT = \tilde{\mathcal{O}}\left(H \Bigg(  
    \sum_{\epindex=1}^\Epindex \frac{L_{\Bar{R}}}{\sqrt{\lsirep}}
    + \E[L_\M] \sqrt{d}
    \sum_{\epindex=1}^\Epindex \frac{L_{\Bar{\transition}}}{\sqrt{\lsipep}}\Bigg)\right).
\end{align*}
\end{restatable}

\
The proof is in the \Cref{proof:subgregret_lsi} and follows from the sub-Gaussian concentration under LSI (\Cref{Eq:conc}).
It implies that PSRL achieves $\BRT=\Otilde\left(\sqrt{d H \totalT}\right)$ if $\lsi_\totalT=\Omega(\totalT)$. In \Cref{sec:application}, we show that this holds for most of the families of distributions studied in literature.

\section{{L{a}PSRL for Approximate Posteriors}}
In practice, the main bottleneck to deploy PSRL is constructing an exact posterior from high-dimensional data and data without any parametric assumption. In these cases, the common strategy is to approximate the posterior distribution computationally efficiently and use it to sample further. But we know that constant approximation error, in terms of $\alpha-$divergence, on the posterior leads to linear regret in the context of Thompson sampling for multi-armed bandits~\citep{NEURIPS2019_f3507289}. This also happens in other RL setups, like contextual bandits and linear MDPs~\citep{simchowitz2021bayesian}. Previous work has noted~\citep{mazumdar2020approximate} that proper decay of this error can allow for sublinear regret in multi-armed bandits. \citeauthor{mazumdar2020approximate,pmlr-v238-zheng24b, pmlr-v202-karbasi23a} designed approximate algorithms for multi-armed bandits and strongly log-concave posteriors. However, \cite{pmlr-v202-karbasi23a,ishfaq2023provable,haque2024more} include planning over episodes, which is essential for MDPs, but the regret analysis heavily depends on strong log-concavity and linearity assumptions. We aim to show that this philosophy of constructing approximate posteriors with proper concentration rates can be applied also to MDPs and with only isoperimetry (LSI in Definition~\ref{ass:lsi}) instead of log-concavity. To start, we derive \Cref{thm:approximate} to control the error rate of concentration of posteriors in RL. 
  \begin{restatable}{theorem}{approxregret}
    \label{thm:approximate}
		Let us sample an MDP from an approximate posterior $\Q_\epindex$ in episode $\epindex$ and use it for planning.
        If $\bel_\epindex$ is the true posterior at $\epindex$ and $\min(\KL{(\bel_\epindex}{\Q_\epindex},\KL{\Q_\epindex}{\bel_\epindex}) \le \epslange$, then regret in an episode due to the approximate posterior is $\mathcal{O}(HB_R\sqrt{\epslange})$. 
      \end{restatable}
    This holds because KL-divergence of posterior controls the growth of Bayesian regret. Proof is in \Cref{proof:approxregret}.

   \begin{restatable}{corollary}{complexity}\label{cor:linear_contract}
		If an algorithm incurs $\widetilde{\mathcal{O}}(\sqrt{T}g(H,\mathcal{S},\mathcal{A}))$ regret for the true posteriors, it will incur the same order of regret for the approximate posteriors if $\epslange \le C \frac{g(H,\mathcal{S},\mathcal{A})^2}{\epindex\Delta_{\text{max}}^2}$ for some $C > 0$.  Here, $\Delta_{\max}\defn\max_{\pol} V_{\pi,1}^{\M}(s_1)-\min_{\pol} V_{\pi,1}^{\M}(s_1) \leq 2 H B_R$ is maximal regret in an episode. 
     \end{restatable}
   Thus, \Cref{cor:linear_contract} states that if the approximation error of the posterior distribution decays linearly with the number of episodes ($l$), then we can achieve $\Otilde\left(\sqrt{T}\right)$ regret by running PSRL with such posteriors. %

\textbf{LaPSRL.} With these results in mind, we design an algorithm, Langevin PSRL (LaPSRL). LaPSRL is detailed in \Cref{alg:lapsrl} and its sampling routine is in \Cref{alg:sample}. LaPSRL extends the PSRL template. In each episode $\epindex$, a tolerable error $\epslange$ is calculated. Then we use SARAH-LD to sample a $\Me\in \R^D$. Depending on the task at hand, SARAH-LD calculates the required step size and learning rate to reach the acceptable error in KL distance, and finally returns a desired sample. This sample is used to obtain an optimal policy, which is then played for that episode. We have two options for initializing the sampling in each episode, from the prior or reusing the previous sample (chained sample setting). Parametrisation of MDP $\Me\in R^\dparam$ is not restricted but the LSI constant depends on it. %
    \begin{algorithm}[t!]
    \caption{Langevin PSRL (LaPSRL)}\label{alg:lapsrl}
    \begin{algorithmic}[1]
    \STATE \textbf{Input}: Prior $\prior(\M)$, Horizon $H$, Regret order $g(H,\mathcal{S},\mathcal{A})$, Likelihood function $\lik(x|\M)$.
    \FOR{$\epindex=1,2,\ldots$}
    \STATE $\epslange=\frac{ g(H,\mathcal{S},\mathcal{A}) }{\epindex \Delta_{\text{max}}^2 }$
    \IF{Chained sampling}
    \STATE $\rho_0=\M_{\epindex-1}$~~~$\rhd${ Reuse last sample from step $l-1$}
    \STATE $\Me \sim \textsc{LS}(\lik(x \mid \M), \hat{\bel}(\M|\datatm), \datat, \epslange, \rho_0)$
    \ELSE
    \STATE $\rho_0 \sim \prior(\M)$~~~$\rhd${ Resample from prior}
    \STATE $\Me \sim \textsc{LS}(\lik(x \mid \M), \prior(\M), \datat, \epslange, \rho_0)$
    \ENDIF
    \STATE Play $\pi^*(\Me)$ until horizon $H$ obtaining data $\datatp=\datat\cup \{(s_i, a_i)\}_{i=H(\epindex-1)+1}^{H\epindex}$.
    \ENDFOR
    \end{algorithmic}
    \end{algorithm}

\textbf{Gradient Complexity.} By combining \Cref{thm:approximate} with log-Sobolev theory and SARAH-LD, we obtain order-optimal Bayesian regret  for any log-Sobolev posterior while limiting the computational gradient complexity of each episode to a low degree polynomial. Gradient complexity signifies the number of gradient steps $\nabla_\M \lik(x|M)$ need to be performed.

\begin{restatable}{corollary}{samplecomplexity}
        \label{corr:sample_complexity}
		For a posterior fulfilling the \Cref{ass:smooth} and \Cref{ass:lsi}, LaPSRL obtains the same order of regret as PSRL while SARAH-LD incurs a gradient complexity %
    $\Otilde\left(\frac{H^3 \epindex^3 L^2}{\lsiep^2} + \frac{dB_R^2 H^{4.5}\epindex^{3.5}L^2}{\lsiep^2 g(H,\mathcal{S},\mathcal{A})^2 }\right) $ in episode $\epindex$.
\end{restatable}

\begin{restatable}{lemma}{regretcomplexity}\label{lemma:regret_lin_growth}
For LSI posteriors (as in \Cref{thm:subgregret_lsi}) with linearly growing LSI constants $\lsiep = \Omega(H\epindex)$, the total gradient complexity of LaPSRL is $\Otilde\left(\tau T + \Epindex \totalT^{1.5}/\sqrt{d}\right)$ that yields regret $\mathcal{O}(\sqrt{dHT})$.%
\end{restatable}

Proofs are deferred to \Cref{proof:samplecomplexity,proof:applycomplex}. This implies that the computational complexity is sub-quadratic for the examples in Section~\ref{sec:application}. %
This could be improved by $\Epindex$ using exponentially increasing episode lengths, requiring modifications to the regret proofs. This result can be compared with \citep{haque2024more} that obtains a total sample complexity of $\Otilde(T^2/\sqrt{d})$, but the analysis is instantiated only for the linear MDPs with an additional $H\sqrt{d}$ term in the regret compared to LaPSRL. Note that the regret obtained by \citeauthor{haque2024more} is expected regret, but this can be transformed into Bayesian regret of the same order.

    \begin{algorithm}[t!]
    \caption{\textsc{Langevin sample (LS)}}%
    \label{alg:sample}
    \begin{algorithmic}[1]
    \STATE \textbf{Input}:  Prior/posterior $\bel(\M)$, data $\datat$, acceptable error $\epslange$, initial sample  $\rho_0$, Likelihood function $\lik(x|\M)$.
    \STATE Learning rate $\eta_\epindex\leftarrow \min\left(\frac{\lsiep}{16 \sqrt{2} L^2 |\datat|^{3/2}} , \frac{3 \lsiep \epslange}{320 d L^2 |\datat|}\right)$
    \STATE \#Steps $k_\epindex \leftarrow \frac{|\datat|}{\lsiep \eta_\epindex} \log \frac{2 \KL{\bel(M)}{\bel\left(\M\mid \datat\right)}}{\epslange}$
    \STATE \textbf{Return} $\Me \leftarrow\textsc{SARAH-LD}\left(\lik(x|\M),\datat, \bel(\M),k_\epindex, \eta_\epindex\right)$
    \end{algorithmic}
    \end{algorithm}

\textbf{Chained Samples.}  The sample complexity to achieve an $\epsilon$ approximation $\hat{\nu}$ of $\nu$ is controlled by $\KL{\hat{\nu}}{\nu}$. The na\"ive approach is sampling only from a prior $\prior$, such as an isotropic Gaussian, and the dependence is only logarithmic in $\KL{\prior}{\nu}$. An alternative is to use the final sample from the previous time step as initialization for the next one. This can be seen as sampling $\rho_0$ from the $\epslangem$-approximate posterior $\hat{\bel}(\M|\datatm)$. This allows for a more practical algorithm as it might be easier to shrink the divergence between two consecutive posteriors than between the prior and a posterior. We show that reusing samples bounds the KL distance to a function of the variance of $\M$. 
      \begin{restatable}{theorem}{chained}
          If $\nabla_z \log \lik(z|\M)$ is $L_z$-Lipschitz and $\lsi_z$-Log Sobolev, with $z$ being the data corresponding to an episode,  
          
          \begin{align}\E_{\lik(z\mid \datat)}\KL{\hat{\bel}\left(\M|\datat\right)}{\bel(\M |\datatp)} \le \epslange + \frac{{L_z}^2}{\lsiz} \epslange^2\mathrm{Var}(\bel\left(\M|\datat\right))  +  \frac{{L_z}^2}{2\lsi_z}\mathrm{Var}(\hat{\bel}\left(\M|\datat\right)),\end{align}  where $\mathrm{Var}(P)$ is the variance of the  distribution $P$. Note that LaPSRL ensures $\epslange =\mathcal{O}(1/l)$. %
      \end{restatable}
    Chaining correlates the sampled parameters.  Since Bayesian regret considers expectation, this does not affect the order of the regret. Rest is to bound the two variance terms. The variance of the true posterior $\mathrm{Var}(\bel\left(\M|\datat\right))$ is bounded by $(1-s)^l\mathrm{Var}(\prior\left(\M\right))$, where $s \in (0,1)$ is the posterior contraction rate of the posterior sampling~\citep{ghosal1997review,wang2023posterior,mou2019diffusion}.  The remaining challenge is to control the variance of the approximate posterior $\hat{\bel}\left(\M|\datat\right)$. But in practice, %
    we know that the variance of the posterior distributions tends to decay as more data is observed, ensuring decay in $\KL{\hat{\bel}\left(\M|\datat\right)}{\bel(\M \mid \datatp)}$. %
    Specifically, for locally and globally log-concave distributions, it is known to decay with $1/l$ for Langevin samplers~\citep{mou2019diffusion,pmlr-v238-zheng24b}. But providing such a control for LSI demands analysing Langevin samplers independent of RL, which we defer to future work. 
    In practice, we achieve control over posterior concentration and thus, regret of LaPSRL with/without chained samples.

\vspace*{-.5em}\section{{Distributions with Linear LSI Constants}}\label{sec:application}\vspace*{-.5em}
Now, we study LSI constants for families of distributions and apply \Cref{thm:subgregret_lsi} to calculate the Bayesian regret of PSRL for corresponding posteriors.

\textbf{Univariate Gaussian.}  For illustration, we calculate the LSI constants for a Gaussian posterior with known variance $\sigma^2$. Here, we assume a Gaussian $(0,\sigma_0^2)$ prior over the mean $\mu$.
Then, $\bel(\mu | \data_n) \propto  \exp\{-\left(\sum_{i=1}^n  \left( \frac{\mu^2}{2n{\sigma_0}^2} + \frac{(\mu-x_i)^2}{2\sigma^2} \right)\right)\}$. 
Then, we have $\gamma = n$, $f_i(\mu)=\left( \frac{\mu^2}{2n{\sigma_0}^2} + \frac{(\mu-x_i)^2}{2\sigma^2} \right)$. Note that $\nabla_\mu^2 f_i(\mu) = \frac{1}{n \sigma_0^2} + \frac{1}{\sigma^2} \le L$. Finally, we use \Cref{theorem:bakry-emery} to calculate $\lsi$. Since $||\nabla_\mu^2 f_i(\mu)||$ is independent of $i$, we can see that $\nabla_\mu^2 -\log \bel(\mu|\data_n) = \nabla_\mu^2 \sum_{i=1}^n f_i(\mu) = \frac{1}{\sigma_0^2} + \frac{n}{\sigma^2}$, which gives $\lsi=\frac{1}{\sigma_0^2} + \frac{n}{\sigma^2}= \Theta(n)$. 
    \begin{restatable}{corollary}{gaussregret}
        PSRL and LaPSRL obtain $BR(T) =\Otilde\left(\sqrt{T\sigma^2}\right)$ with univariate Gaussian posteriors.
    \end{restatable}

\textbf{Log-concave and Mixture of Log-concave Distributions.}
\begin{restatable}{theorem}{lsigeneral}\label{thm:lsigeneral}
 (a) Any log-concave posterior fulfils LSI with $\alpha_n=\Theta(n)$. 
 (b) Any posterior that is a mixture of $k$ log-concave distributions has $\lsi^{\text{Mixture}}_n = \Omega\left(\frac{ n \min p_i}{4k(1-\log(\min p_i))}\right)$.
\end{restatable}
This result comes from the superadditivity of minimum eigenvalues of Hessians, and thus, LSI constants for log-concave distributions. The result for mixtures follows from \Cref{theorem:mixturelsi} (\Cref{proof:lsigeneral}). 
Combining \Cref{thm:subgregret_lsi} and \Cref{thm:lsigeneral}, we obtain the sublinear regret bound.
\begin{corollary}
    If $|\Bar{R}_\M(s)|\le B_R \forall s$, any log-concave posterior over MDPs $\M$ yields $\BRT=\Otilde\left(\sqrt{T}\right)$ for PSRL.
    Under the same conditions, PSRL obtains for any posterior that is a mixture of $k$ log-concave posteriors with non-zero overlap 
    $\BRT=\Otilde\left(\sqrt{\frac{4kT}{\min p_i}}\right)$.
\end{corollary}

\textbf{The Final Goal: General Log-Sobolev Distributions.}
\begin{restatable}{theorem}{boundedlik}\label{thm:boundedlik}
    A log-Sobolev distribution with bounded likelihood ratio: $|\log \frac{\lik(x\mid \M)}{\lik(x \mid \M')}| \le \Gamma$, has a log-Sobolev posterior.
\end{restatable}
This result is interesting because it means that a very wide family of settings have log-Sobolev posteriors.
Unfortunately, we have been unable to prove that the log-Sobolev constant of a posterior, under some suitable conditions, will always scale as $\Omega(n)$. Although we conjecture that this could be possible. Similar assumptions of linear growth of constants are made in works involving other isoperimetric ineuqalities, namely Poincar\'e inequality, but in the offline setting~\citep{haddouche2024pac}. This also matches the intuition from the asymptotic results of the Bernstein--von Mises theorem which gives a log-Sobolev constant of $\Theta(n)$ as $n \to \infty$. %
For now, in these settings, we can rely on the results in  \Cref{thm:subgregret} exhibiting sublinear regret.

\section{{Experimental Analysis}}
 \begin{figure*}[t!]
        \centering%
        \begin{subfigure}[t]{0.32\linewidth}
        \centering
        \includegraphics[width=\linewidth]{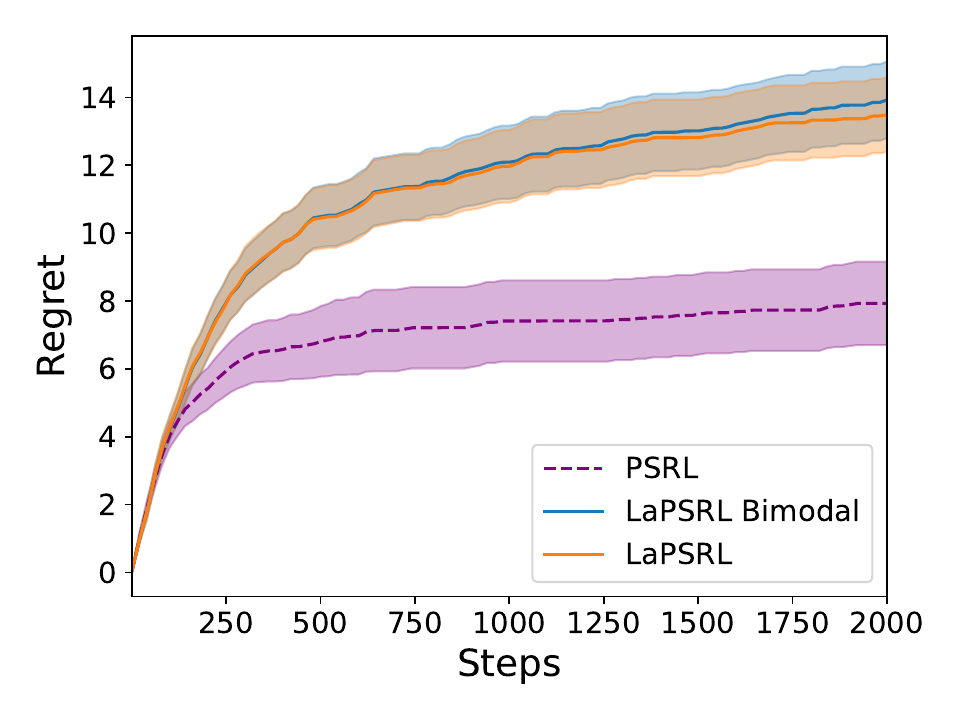}%
        \caption{Gaussian Bandit}\label{fig:bandit}
    \end{subfigure}\hfill
    \begin{subfigure}[t]{0.32\linewidth}
        \centering
\includegraphics[width=\linewidth]{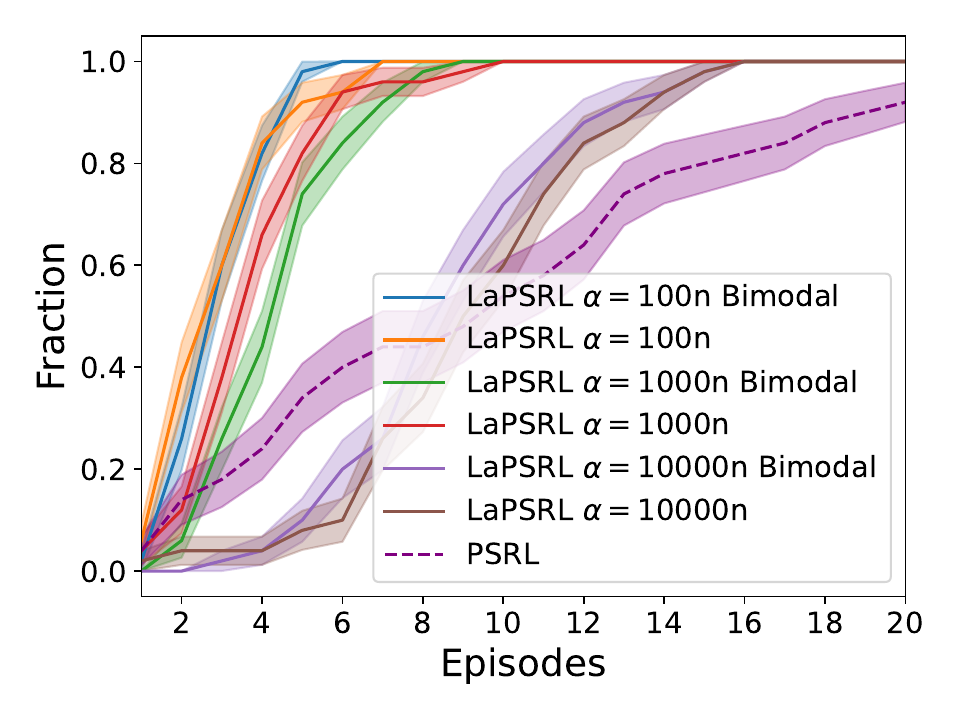}%
        \caption{Cartpole}\label{fig:cartpole}
    \end{subfigure}\hfill
        \begin{subfigure}[t]{0.32\linewidth}
        \centering
\includegraphics[width=\linewidth]{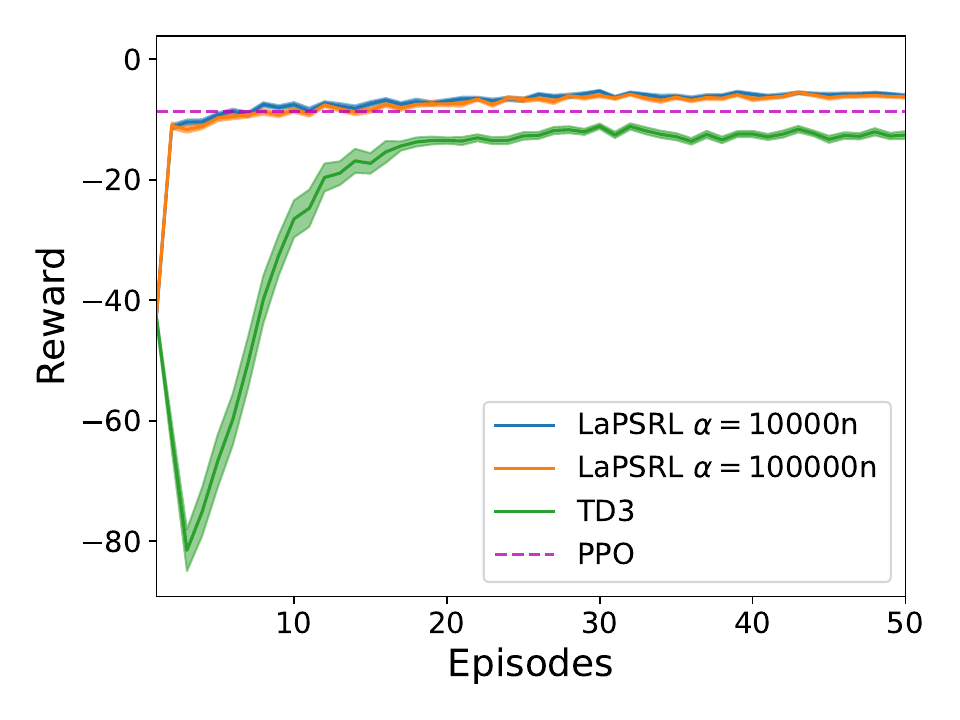}%
        \caption{Reacher}\label{fig:reacher}
    \end{subfigure}%
    \caption{We compare LaPSRL against baselines. In the bandit and Cartpole experiments, we benchmark with PSRL, and in Reacher with TD3 and PPO. For the Gaussian bandits, we compare the expected regret and for Cartpole we evaluate how many episodes it takes to solve the task. Finally, in Reacher, we study the average regret per episode. In all environments, we average over 50 runs with the standard error highlighted around the average. Larger plots are in  \Cref{fig:app_plots}.}\label{fig_posterior}\vspace*{-1em}
    \end{figure*}
We run a set of experiments on  two environments to verify that the LaPSRL is competitive. 
The goal of this section is to answer the following questions:
\begin{center}
    \emph{Can LaPSRL work for varying domains and settings?}\\ \emph{Is LaPSRL efficient in terms of performance?} 
\end{center}

To demonstrate this we perform a set of experiments.
 First, we deploy LaPSRL on a Gaussian multi-armed bandit task with two arms. Second, we perform experiments with a LQR  setup on the Cartpole  environment and finally with a neural network on the Reacher environment. On Cartpole and Bandits, we  additionally test a bimodal prior over the arms to demonstrate a non log-concave setting.  Additional experimental details are found in the appendix. 

\textbf{Gaussian Batched Bandits.} We use LaPSRL on a Gaussian multi-armed bandit task with two arms. To preserve computations, we use a batched approach, such that each action is taken 20 times each time it is sampled. As a baseline, we compare with the performance of  PSRL from the true posterior.   The results can be seen in \Cref{fig:bandit} where we plot the expected regret.
 We can see that both LaPSRL algorithms perform  similarly, and follow the approximate shape of  the theoretical priors. This is in line with the theoretical results of order optimality.
 
\textbf{Continuous MDPs.} We evaluate LaPSRL on two continuous environments, Cartpole and Reacher. For both experiments, we try different settings for $\alpha$ (See appendix). 
The Cartpole environment~\citep{6313077} is modified to have continuous action and we use parameterize the model as a Linear Quadratic Regulator model \citep{kalman1960new}.  
We use a PSRL algorithm which samples from Bayesian linear regression priors~\citep{minka2000bayesian} as a baseline. 
The results from this experiment can be found in \Cref{fig:cartpole} where we plot what fraction of the runs have solved the task (i.e. taking 200 steps without failing). Here we see that all versions successfully handle the task, even faster than the PSRL baseline. It also takes longer for the experiments with larger $\lsi$ values to converge. The slower convergence of PSRL can be due to the priors not being the same. %

 Reacher is a standard environment from the Gymnasium library~\citep{towers2024gymnasium}.
 Here we use a neural network model, this is not necessarily log-Sobolev, but we wish to show that this approach still is useful. Here, we benchmark with TD3~\citep{fujimoto2018addressing} and PPO~\citep{schulman2017proximalpolicyoptimizationalgorithms}. \Cref{fig:reacher} shows that LaPSRL learns very quickly, while TD3 is unable to learn a policy that is equally good in a similar time frame. Due to no experience replay, PPO is significantly less data efficient. Thus, we present its average performance finally after 7000 episodes. %

\textbf{Results and Discussions.} To conclude, we find that in all three experiments, LaPSRL works well. These settings are quite different, and tell us that LaPSRL works efficiently in very different settings, supporting the claim that LaPSRL can perform well for varying domains and settings.

\section{{Extended Related Works}}%

In addition to the works discussed previously, we present an overview of the related works with approximate posteriors.

Thompson sampling often requires approximations when exact posterior calculations or sampling become intractable~\citep{wang2023model, pmlr-v202-sasso23a, pmlr-v216-osband23a}. While various algorithms for approximate sampling exist, they generally lack regret guarantees. %
\citet{huang2023optimal} does provide Bayesian regret bounds for approximate upper confidence bounds in bandits.

Recent works have extensively explored Langevin methods in bandits and RL~\citep{kim2023learning, dwaracherla2020langevin, yamamoto2023mean, anonymous2025langevin} but without regret guarantees. Further studies have applied these methods to specific settings: offline RL~\citep{nguyen-tang2024posterior} and inverse RL~\citep{JMLR:v22:20-625}. 
Similarly, \citet{hsu2024randomized} establish regret bounds for multi-agent RL, but only under linear function approximations. \citep{kim2024approximate} do the same for LQR systems under strongly log-concave assumptions.
In contrast, our work provides a more general framework that overcomes these restrictive and eclectic assumptions of linearity, discreteness, or log-concavity.
To conclude, there are works looking into using Langevin methods for RL but \emph{none that comes with regret guarantees for a setting as general as LSI distributions}.

More work on Bayesian reinforcement learning includes \citep{o2018uncertainty,osband:thompson:nips:2013,pmlr-v216-osband23a,pmlr-v137-jorge20a,dimitrakakis:mmbi:ewrl:2011,NEURIPS2021_edb446b6,NEURIPS2021_7180cffd,pmlr-v206-luis23a,luis2024modelbasedepistemicvariancevalues,JMLR:v25:23-0913,zintgraf2021varibad,NEURIPS2018_3de568f8,pmlr-v139-fan21b,pmlr-v180-eriksson22a,Ghavamzadeh_2015,cronrath2018bagger, dearden1999model,Reisinger2008}.
A few of these works come with regret results, but they are limited to cases with easy posterior updates such as linear models or discrete state and actions spaces~\citep{grover2020bayesian}.

\textit{}\section{{Discussion \& Future Works}}%
In this paper, we aim to understand whether we can design algorithms with sublinear regret for any isoperimetric distribution. 
We specifically study PSRL type algorithms for posteriors satisfying log-Sobolev inequalities. We show that if we can compute exact posteriors and sample from them, PSRL can achieve $\Otilde(\sqrt{T})$ regret in an episodic MDP under log-Sobolev and some additional mild assumptions, this extends the setting where such results exist.
We further design a generic Langevin sampling based extension of PSRL, namely LaPSRL.
We show that LaPSRL also achieves $\Otilde(\sqrt{T})$ in these settings. %
We plug-in SARAH-LD as the Langevin sampling algorithm, and derive upper bounds on the required gradient complexity and chained sample complexity.
Finally, we test LaPSRL in bandit and continuous MDP environments.
We show that the variants of LaPSRL perform competitively with respect to baselines in all these settings.

In the future, it will be interesting to extend LaPSRL's analysis to neural tangent kernel's yielding a better understanding of deep RL.%

\subsubsection*{Acknowledgements}
This work was partially supported by the Wallenberg AI, Autonomous Systems and Software Program (WASP) funded by the Knut and Alice Wallenberg Foundation, and the Norwegian Research Council Project "Algorithms and Models for Socially Beneficial AI".
D. Basu acknowledges the Inria-Kyoto University Associate Team ``RELIANT'' for supporting the project, the CHIST-ERA CausalXRL project, the ANR JCJC for the REPUBLIC project (ANR-22-CE23-0003-01), and the PEPR project FOUNDRY (ANR23-PEIA-0003).
The computations were enabled by resources provided by the National Academic Infrastructure for Supercomputing in Sweden (NAISS), partially funded by the Swedish Research Council through grant agreement no. 2022-06725. The authors also wish to thank Hannes Eriksson for his assistance.

\bibliographystyle{icml2025}
\bibliography{references}

\begin{thebibliography}{94}
\providecommand{\natexlab}[1]{#1}
\providecommand{\url}[1]{\texttt{#1}}
\expandafter\ifx\csname urlstyle\endcsname\relax
  \providecommand{\doi}[1]{doi: #1}\else
  \providecommand{\doi}{doi: \begingroup \urlstyle{rm}\Url}\fi

\bibitem[Abeille \& Lazaric(2018)Abeille and Lazaric]{pmlr-v80-abeille18a}
Abeille, M. and Lazaric, A.
\newblock Improved regret bounds for thompson sampling in linear quadratic
  control problems.
\newblock In Dy, J. and Krause, A. (eds.), \emph{Proceedings of the 35th
  International Conference on Machine Learning}, volume~80 of \emph{Proceedings
  of Machine Learning Research}, pp.\  1--9. PMLR, 10--15 Jul 2018.
\newblock URL \url{http://proceedings.mlr.press/v80/abeille18a.html}.

\bibitem[Agrawal \& Jia(2017)Agrawal and Jia]{NIPS2017_3621f145}
Agrawal, S. and Jia, R.
\newblock Optimistic posterior sampling for reinforcement learning: worst-case
  regret bounds.
\newblock In Guyon, I., Luxburg, U.~V., Bengio, S., Wallach, H., Fergus, R.,
  Vishwanathan, S., and Garnett, R. (eds.), \emph{Advances in Neural
  Information Processing Systems}, volume~30. Curran Associates, Inc., 2017.
\newblock URL
  \url{https://proceedings.neurips.cc/paper/2017/file/3621f1454cacf995530ea53652ddf8fb-Paper.pdf}.

\bibitem[Anonymous(2025)]{anonymous2025langevin}
Anonymous.
\newblock Langevin soft actor-critic: Efficient exploration through
  uncertainty-driven critic learning.
\newblock In \emph{The Thirteenth International Conference on Learning
  Representations}, 2025.
\newblock URL \url{https://openreview.net/forum?id=FvQsk3la17}.

\bibitem[Azar et~al.(2017)Azar, Osband, and Munos]{azar2017minimax}
Azar, M.~G., Osband, I., and Munos, R.
\newblock Minimax regret bounds for reinforcement learning.
\newblock In \emph{International Conference on Machine Learning}, pp.\
  263--272. PMLR, 2017.

\bibitem[Bakry et~al.(2014)Bakry, Gentil, Ledoux, et~al.]{bakry2014analysis}
Bakry, D., Gentil, I., Ledoux, M., et~al.
\newblock \emph{Analysis and geometry of Markov diffusion operators}, volume
  103.
\newblock Springer, 2014.

\bibitem[Barthe \& Kolesnikov(2008)Barthe and Kolesnikov]{barthe2008mass}
Barthe, F. and Kolesnikov, A.~V.
\newblock Mass transport and variants of the logarithmic sobolev inequality.
\newblock \emph{Journal of Geometric Analysis}, 18\penalty0 (4):\penalty0
  921--979, 2008.

\bibitem[Barto et~al.(1983)Barto, Sutton, and Anderson]{6313077}
Barto, A.~G., Sutton, R.~S., and Anderson, C.~W.
\newblock Neuronlike adaptive elements that can solve difficult learning
  control problems.
\newblock \emph{IEEE Transactions on Systems, Man, and Cybernetics},
  SMC-13\penalty0 (5):\penalty0 834--846, 1983.
\newblock \doi{10.1109/TSMC.1983.6313077}.

\bibitem[Bizeul(2023)]{bizeul2023logsobolev}
Bizeul, P.
\newblock On the log-sobolev constant of log-concave measures, 2023.

\bibitem[Boucheron et~al.(2003)Boucheron, Lugosi, and
  Bousquet]{boucheron2003concentration}
Boucheron, S., Lugosi, G., and Bousquet, O.
\newblock Concentration inequalities.
\newblock In \emph{Summer school on machine learning}, pp.\  208--240.
  Springer, 2003.

\bibitem[Bradbury et~al.(2018)Bradbury, Frostig, Hawkins, Johnson, Leary,
  Maclaurin, Necula, Paszke, Vander{P}las, Wanderman-{M}ilne, and
  Zhang]{jax2018github}
Bradbury, J., Frostig, R., Hawkins, P., Johnson, M.~J., Leary, C., Maclaurin,
  D., Necula, G., Paszke, A., Vander{P}las, J., Wanderman-{M}ilne, S., and
  Zhang, Q.
\newblock {JAX}: composable transformations of {P}ython+{N}um{P}y programs,
  2018.
\newblock URL \url{http://github.com/jax-ml/jax}.

\bibitem[Buening et~al.(2023)Buening, Dimitrakakis, Eriksson, Grover, and
  Jorge]{pmlr-v206-buening23a}
Buening, T.~K., Dimitrakakis, C., Eriksson, H., Grover, D., and Jorge, E.
\newblock Minimax-bayes reinforcement learning.
\newblock In Ruiz, F., Dy, J., and van~de Meent, J.-W. (eds.),
  \emph{Proceedings of The 26th International Conference on Artificial
  Intelligence and Statistics}, volume 206 of \emph{Proceedings of Machine
  Learning Research}, pp.\  7511--7527. PMLR, 25--27 Apr 2023.
\newblock URL \url{https://proceedings.mlr.press/v206/buening23a.html}.

\bibitem[Castillo et~al.(2015)Castillo, Schmidt-Hieber, and Van~der
  Vaart]{castillo2015bayesian}
Castillo, I., Schmidt-Hieber, J., and Van~der Vaart, A.
\newblock Bayesian linear regression with sparse priors.
\newblock \emph{The Annals of Statistics}, pp.\  1986--2018, 2015.

\bibitem[Cattiaux et~al.(2010)Cattiaux, Guillin, and Wu]{cattiaux2010note}
Cattiaux, P., Guillin, A., and Wu, L.-M.
\newblock A note on talagrand’s transportation inequality and logarithmic
  sobolev inequality.
\newblock \emph{Probability theory and related fields}, 148:\penalty0 285--304,
  2010.

\bibitem[Chafa\"{i} \& Lehec(2023)Chafa\"{i} and Lehec]{LSIessentials}
Chafa\"{i}, D. and Lehec, J.
\newblock Logarithmic sobolev inequalities essentials, 2023.
\newblock URL
  \url{https://djalil.chafai.net/docs/M2/chafai-lehec-m2-lsie-lecture-notes.pdf}.
\newblock Accessed on 08/10/2024.

\bibitem[Chakraborty et~al.(2023)Chakraborty, Roy, and
  Tewari]{chakraborty2023thompson}
Chakraborty, S., Roy, S., and Tewari, A.
\newblock Thompson sampling for high-dimensional sparse linear contextual
  bandits.
\newblock In \emph{International Conference on Machine Learning}, pp.\
  3979--4008. PMLR, 2023.

\bibitem[Chen et~al.(2021)Chen, Chewi, and Niles-Weed]{CHEN2021109236}
Chen, H.-B., Chewi, S., and Niles-Weed, J.
\newblock Dimension-free log-sobolev inequalities for mixture distributions.
\newblock \emph{Journal of Functional Analysis}, 281\penalty0 (11):\penalty0
  109236, 2021.
\newblock ISSN 0022-1236.
\newblock \doi{https://doi.org/10.1016/j.jfa.2021.109236}.

\bibitem[Chowdhury \& Gopalan(2019)Chowdhury and Gopalan]{chowdhury2019online}
Chowdhury, S.~R. and Gopalan, A.
\newblock Online learning in kernelized markov decision processes, 2019.

\bibitem[Chowdhury et~al.(2021)Chowdhury, Gopalan, and
  Maillard]{pmlr-v130-chowdhury21b}
Chowdhury, S.~R., Gopalan, A., and Maillard, O.-A.
\newblock Reinforcement learning in parametric mdps with exponential families.
\newblock In Banerjee, A. and Fukumizu, K. (eds.), \emph{Proceedings of The
  24th International Conference on Artificial Intelligence and Statistics},
  volume 130 of \emph{Proceedings of Machine Learning Research}, pp.\
  1855--1863. PMLR, 13--15 Apr 2021.

\bibitem[Chua et~al.(2018)Chua, Calandra, McAllister, and
  Levine]{NEURIPS2018_3de568f8}
Chua, K., Calandra, R., McAllister, R., and Levine, S.
\newblock Deep reinforcement learning in a handful of trials using
  probabilistic dynamics models.
\newblock In Bengio, S., Wallach, H., Larochelle, H., Grauman, K.,
  Cesa-Bianchi, N., and Garnett, R. (eds.), \emph{Advances in Neural
  Information Processing Systems}, volume~31. Curran Associates, Inc., 2018.

\bibitem[Cronrath et~al.(2018)Cronrath, Jorge, Moberg, Jirstrand, and
  Lennartson]{cronrath2018bagger}
Cronrath, C., Jorge, E., Moberg, J., Jirstrand, M., and Lennartson, B.
\newblock Bagger: A bayesian algorithm for safe and query-efficient imitation
  learning.
\newblock 2018.

\bibitem[Dearden et~al.(1999)Dearden, Friedman, and Andre]{dearden1999model}
Dearden, R., Friedman, N., and Andre, D.
\newblock Model based {B}ayesian exploration.
\newblock In \emph{Proceedings of the Fifteenth conference on Uncertainty in
  artificial intelligence}, pp.\  150--159, 1999.

\bibitem[Dimitrakakis(2011)]{dimitrakakis:mmbi:ewrl:2011}
Dimitrakakis, C.
\newblock Robust {Bayesian} reinforcement learning through tight lower bounds.
\newblock In \emph{European Workshop on Reinforcement Learning (EWRL 2011)},
  pp.\  177--188, 2011.

\bibitem[Dubey et~al.(2016)Dubey, J~Reddi, Williamson, Poczos, Smola, and
  Xing]{svrgld}
Dubey, K.~A., J~Reddi, S., Williamson, S.~A., Poczos, B., Smola, A.~J., and
  Xing, E.~P.
\newblock Variance reduction in stochastic gradient langevin dynamics.
\newblock \emph{Advances in neural information processing systems}, 29, 2016.

\bibitem[Dwaracherla \& Van~Roy(2020)Dwaracherla and
  Van~Roy]{dwaracherla2020langevin}
Dwaracherla, V. and Van~Roy, B.
\newblock Langevin dqn.
\newblock \emph{arXiv preprint arXiv:2002.07282}, 2020.

\bibitem[Efron(2022)]{efron2022exponential}
Efron, B.
\newblock \emph{Exponential families in theory and practice}.
\newblock Cambridge University Press, 2022.

\bibitem[Eriksson et~al.(2022)Eriksson, Basu, Alibeigi, and
  Dimitrakakis]{pmlr-v180-eriksson22a}
Eriksson, H., Basu, D., Alibeigi, M., and Dimitrakakis, C.
\newblock Sentinel: taming uncertainty with ensemble based distributional
  reinforcement learning.
\newblock In Cussens, J. and Zhang, K. (eds.), \emph{Proceedings of the
  Thirty-Eighth Conference on Uncertainty in Artificial Intelligence}, volume
  180 of \emph{Proceedings of Machine Learning Research}, pp.\  631--640. PMLR,
  01--05 Aug 2022.
\newblock URL \url{https://proceedings.mlr.press/v180/eriksson22a.html}.

\bibitem[Fan \& Ming(2021)Fan and Ming]{pmlr-v139-fan21b}
Fan, Y. and Ming, Y.
\newblock Model-based reinforcement learning for continuous control with
  posterior sampling.
\newblock In Meila, M. and Zhang, T. (eds.), \emph{Proceedings of the 38th
  International Conference on Machine Learning}, volume 139 of
  \emph{Proceedings of Machine Learning Research}, pp.\  3078--3087. PMLR,
  18--24 Jul 2021.
\newblock URL \url{http://proceedings.mlr.press/v139/fan21b.html}.

\bibitem[Fellows et~al.(2021)Fellows, Hartikainen, and
  Whiteson]{NEURIPS2021_7180cffd}
Fellows, M., Hartikainen, K., and Whiteson, S.
\newblock Bayesian bellman operators.
\newblock In Ranzato, M., Beygelzimer, A., Dauphin, Y., Liang, P., and Vaughan,
  J.~W. (eds.), \emph{Advances in Neural Information Processing Systems},
  volume~34, pp.\  13641--13656. Curran Associates, Inc., 2021.
\newblock URL
  \url{https://proceedings.neurips.cc/paper_files/paper/2021/file/7180cffd6a8e829dacfc2a31b3f72ece-Paper.pdf}.

\bibitem[Ferguson(1962)]{ferguson1962location}
Ferguson, T.~S.
\newblock Location and scale parameters in exponential families of
  distributions.
\newblock \emph{The Annals of mathematical statistics}, 33\penalty0
  (3):\penalty0 986--1001, 1962.

\bibitem[Fujimoto et~al.(2018)Fujimoto, Hoof, and
  Meger]{fujimoto2018addressing}
Fujimoto, S., Hoof, H., and Meger, D.
\newblock Addressing function approximation error in actor-critic methods.
\newblock In \emph{International conference on machine learning}, pp.\
  1587--1596. PMLR, 2018.

\bibitem[Geramifard et~al.(2013)Geramifard, Walsh, Tellex, Chowdhary, Roy, and
  How]{MAL-042}
Geramifard, A., Walsh, T.~J., Tellex, S., Chowdhary, G., Roy, N., and How,
  J.~P.
\newblock A tutorial on linear function approximators for dynamic programming
  and reinforcement learning.
\newblock \emph{Foundations and Trends® in Machine Learning}, 6\penalty0
  (4):\penalty0 375--451, 2013.
\newblock ISSN 1935-8237.
\newblock \doi{10.1561/2200000042}.

\bibitem[Ghavamzadeh et~al.(2015)Ghavamzadeh, Mannor, Pineau, and
  Tamar]{Ghavamzadeh_2015}
Ghavamzadeh, M., Mannor, S., Pineau, J., and Tamar, A.
\newblock Convex optimization: Algorithms and complexity.
\newblock \emph{Foundations and Trends® in Machine Learning}, 8\penalty0
  (5–6):\penalty0 359–483, 2015.
\newblock ISSN 1935-8245.
\newblock \doi{10.1561/2200000049}.
\newblock URL \url{http://dx.doi.org/10.1561/2200000049}.

\bibitem[Ghosal(1997)]{ghosal1997review}
Ghosal, S.
\newblock A review of consistency and convergence of posterior distribution.
\newblock In \emph{Varanashi Symposium in Bayesian Inference, Banaras Hindu
  University}. Citeseer, 1997.

\bibitem[Grover et~al.(2020)Grover, Basu, and Dimitrakakis]{grover2020bayesian}
Grover, D., Basu, D., and Dimitrakakis, C.
\newblock Bayesian reinforcement learning via deep, sparse sampling.
\newblock In \emph{International Conference on Artificial Intelligence and
  Statistics}, pp.\  3036--3045. PMLR, 2020.

\bibitem[Haddouche et~al.(2024)Haddouche, Viallard, Simsekli, and
  Guedj]{haddouche2024pac}
Haddouche, M., Viallard, P., Simsekli, U., and Guedj, B.
\newblock A pac-bayesian link between generalisation and flat minima.
\newblock \emph{arXiv preprint arXiv:2402.08508}, 2024.

\bibitem[Haque et~al.(2024)Haque, Tan, Yang, Lan, Lu, Mahmood, Precup, and
  Xu]{haque2024more}
Haque, I., Tan, Y., Yang, Y., Lan, Q., Lu, J., Mahmood, A.~R., Precup, D., and
  Xu, P.
\newblock More efficient randomized exploration for reinforcement learning via
  approximate sampling.
\newblock \emph{Reinforcement Learning Journal}, 3\penalty0 (1), 2024.

\bibitem[Holley \& Stroock(1987)Holley and Stroock]{holley1986logarithmic}
Holley, R. and Stroock, D.
\newblock Logarithmic sobolev inequalities and stochastic ising models.
\newblock \emph{Journal of Statistical Physics}, 46\penalty0 (5-6):\penalty0
  1159--1194, 1987.

\bibitem[Hsu et~al.(2024)Hsu, Wang, Pajic, and Xu]{hsu2024randomized}
Hsu, H.-L., Wang, W., Pajic, M., and Xu, P.
\newblock Randomized exploration in cooperative multi-agent reinforcement
  learning.
\newblock \emph{arXiv preprint arXiv:2404.10728}, 2024.

\bibitem[Huang et~al.(2022)Huang, Dossa, Ye, Braga, Chakraborty, Mehta, and
  Araújo]{huang2022cleanrl}
Huang, S., Dossa, R. F.~J., Ye, C., Braga, J., Chakraborty, D., Mehta, K., and
  Araújo, J.~G.
\newblock Cleanrl: High-quality single-file implementations of deep
  reinforcement learning algorithms.
\newblock \emph{Journal of Machine Learning Research}, 23\penalty0
  (274):\penalty0 1--18, 2022.
\newblock URL \url{http://jmlr.org/papers/v23/21-1342.html}.

\bibitem[Huang et~al.(2023)Huang, Lam, Meisami, and Zhang]{huang2023optimal}
Huang, Z., Lam, H., Meisami, A., and Zhang, H.
\newblock Optimal regret is achievable with bounded approximate inference
  error: An enhanced bayesian upper confidence bound framework.
\newblock In \emph{Thirty-seventh Conference on Neural Information Processing
  Systems}, 2023.
\newblock URL \url{https://openreview.net/forum?id=vwr4bHHsRT}.

\bibitem[Huix et~al.(2023)Huix, Zhang, and Durmus]{pmlr-v206-huix23a}
Huix, T., Zhang, M., and Durmus, A.
\newblock Tight regret and complexity bounds for thompson sampling via langevin
  monte carlo.
\newblock In Ruiz, F., Dy, J., and van~de Meent, J.-W. (eds.),
  \emph{Proceedings of The 26th International Conference on Artificial
  Intelligence and Statistics}, volume 206 of \emph{Proceedings of Machine
  Learning Research}, pp.\  8749--8770. PMLR, 25--27 Apr 2023.
\newblock URL \url{https://proceedings.mlr.press/v206/huix23a.html}.

\bibitem[Ishfaq et~al.(2024)Ishfaq, Lan, Xu, Mahmood, Precup, Anandkumar, and
  Azizzadenesheli]{ishfaq2023provable}
Ishfaq, H., Lan, Q., Xu, P., Mahmood, A.~R., Precup, D., Anandkumar, A., and
  Azizzadenesheli, K.
\newblock Provable and practical: Efficient exploration in reinforcement
  learning via langevin monte carlo.
\newblock In \emph{The Twelfth International Conference on Learning
  Representations}, 2024.
\newblock URL \url{https://openreview.net/forum?id=nfIAEJFiBZ}.

\bibitem[Jin et~al.(2019)Jin, Netrapalli, Ge, Kakade, and
  Jordan]{jin2019shortnoteconcentrationinequalities}
Jin, C., Netrapalli, P., Ge, R., Kakade, S.~M., and Jordan, M.~I.
\newblock A short note on concentration inequalities for random vectors with
  subgaussian norm, 2019.
\newblock URL \url{https://arxiv.org/abs/1902.03736}.

\bibitem[Jorge et~al.(2020)Jorge, Eriksson, Dimitrakakis, Basu, and
  Grover]{pmlr-v137-jorge20a}
Jorge, E., Eriksson, H., Dimitrakakis, C., Basu, D., and Grover, D.
\newblock Inferential induction: A novel framework for {B}ayesian reinforcement
  learning.
\newblock In \emph{Proceedings on "I Can't Believe It's Not Better!" at NeurIPS
  Workshops}, volume 137 of \emph{Proceedings of Machine Learning Research},
  pp.\  43--52. PMLR, 12 Dec 2020.
\newblock URL \url{http://proceedings.mlr.press/v137/jorge20a.html}.

\bibitem[Kalman(1960)]{kalman1960new}
Kalman, R.~E.
\newblock A new approach to linear filtering and prediction problems.
\newblock 1960.

\bibitem[Karbasi et~al.(2023)Karbasi, Kuang, Ma, and
  Mitra]{pmlr-v202-karbasi23a}
Karbasi, A., Kuang, N.~L., Ma, Y., and Mitra, S.
\newblock {L}angevin thompson sampling with logarithmic communication: Bandits
  and reinforcement learning.
\newblock In Krause, A., Brunskill, E., Cho, K., Engelhardt, B., Sabato, S.,
  and Scarlett, J. (eds.), \emph{Proceedings of the 40th International
  Conference on Machine Learning}, volume 202 of \emph{Proceedings of Machine
  Learning Research}, pp.\  15828--15860. PMLR, 23--29 Jul 2023.
\newblock URL \url{https://proceedings.mlr.press/v202/karbasi23a.html}.

\bibitem[Kim(2023)]{kim2023learning}
Kim, G.
\newblock Learning linear-quadratic regulators via thompson sampling with
  preconditioned langevin dynamics, 2023.

\bibitem[Kim et~al.(2024)Kim, Kim, and Yang]{kim2024approximate}
Kim, Y., Kim, G., and Yang, I.
\newblock Approximate thompson sampling for learning linear quadratic
  regulators with $o (\sqrt\{T\}) $ regret.
\newblock \emph{arXiv preprint arXiv:2405.19380}, 2024.

\bibitem[Kinoshita \& Suzuki(2022)Kinoshita and Suzuki]{kinoshita2022improved}
Kinoshita, Y. and Suzuki, T.
\newblock Improved convergence rate of stochastic gradient langevin dynamics
  with variance reduction and its application to optimization.
\newblock In Oh, A.~H., Agarwal, A., Belgrave, D., and Cho, K. (eds.),
  \emph{Advances in Neural Information Processing Systems}, 2022.
\newblock URL \url{https://openreview.net/forum?id=Sj2z__i1wX-}.

\bibitem[Koehler \& Vuong(2024)Koehler and Vuong]{koehler2024sampling}
Koehler, F. and Vuong, T.-D.
\newblock Sampling multimodal distributions with the vanilla score: Benefits of
  data-based initialization.
\newblock In \emph{The Twelfth International Conference on Learning
  Representations}, 2024.
\newblock URL \url{https://openreview.net/forum?id=oAMArMMQxb}.

\bibitem[Koehler et~al.(2023)Koehler, Heckett, and
  Risteski]{koehler2023statistical}
Koehler, F., Heckett, A., and Risteski, A.
\newblock Statistical efficiency of score matching: The view from isoperimetry.
\newblock In \emph{The Eleventh International Conference on Learning
  Representations}, 2023.
\newblock URL \url{https://openreview.net/forum?id=TD7AnQjNzR6}.

\bibitem[Krishnamurthy \& Yin(2021)Krishnamurthy and Yin]{JMLR:v22:20-625}
Krishnamurthy, V. and Yin, G.
\newblock Langevin dynamics for adaptive inverse reinforcement learning of
  stochastic gradient algorithms.
\newblock \emph{Journal of Machine Learning Research}, 22\penalty0
  (121):\penalty0 1--49, 2021.
\newblock URL \url{http://jmlr.org/papers/v22/20-625.html}.

\bibitem[Kuang et~al.(2023)Kuang, Yin, Wang, Wang, and Ma]{kuang2023posterior}
Kuang, N.~L., Yin, M., Wang, M., Wang, Y.-X., and Ma, Y.
\newblock Posterior sampling with delayed feedback for reinforcement learning
  with linear function approximation.
\newblock In \emph{Thirty-seventh Conference on Neural Information Processing
  Systems}, 2023.
\newblock URL \url{https://openreview.net/forum?id=RiyH3z7oIF}.

\bibitem[Lattimore \& Szepesv{\'a}ri(2020)Lattimore and
  Szepesv{\'a}ri]{lattimore2020bandit}
Lattimore, T. and Szepesv{\'a}ri, C.
\newblock \emph{Bandit algorithms}.
\newblock Cambridge University Press, 2020.

\bibitem[Ledoux(2006)]{ledoux2006concentration}
Ledoux, M.
\newblock Concentration of measure and logarithmic sobolev inequalities.
\newblock In \emph{Seminaire de probabilites XXXIII}, pp.\  120--216. Springer,
  2006.

\bibitem[Luis et~al.(2023)Luis, Bottero, Vinogradska, Berkenkamp, and
  Peters]{pmlr-v206-luis23a}
Luis, C.~E., Bottero, A.~G., Vinogradska, J., Berkenkamp, F., and Peters, J.
\newblock Model-based uncertainty in value functions.
\newblock In Ruiz, F., Dy, J., and van~de Meent, J.-W. (eds.),
  \emph{Proceedings of The 26th International Conference on Artificial
  Intelligence and Statistics}, volume 206 of \emph{Proceedings of Machine
  Learning Research}, pp.\  8029--8052. PMLR, 25--27 Apr 2023.
\newblock URL \url{https://proceedings.mlr.press/v206/luis23a.html}.

\bibitem[Luis et~al.(2024{\natexlab{a}})Luis, Bottero, Vinogradska, Berkenkamp,
  and Peters]{JMLR:v25:23-0913}
Luis, C.~E., Bottero, A.~G., Vinogradska, J., Berkenkamp, F., and Peters, J.
\newblock Value-distributional model-based reinforcement learning.
\newblock \emph{Journal of Machine Learning Research}, 25\penalty0
  (298):\penalty0 1--42, 2024{\natexlab{a}}.
\newblock URL \url{http://jmlr.org/papers/v25/23-0913.html}.

\bibitem[Luis et~al.(2024{\natexlab{b}})Luis, Bottero, Vinogradska, Berkenkamp,
  and Peters]{luis2024modelbasedepistemicvariancevalues}
Luis, C.~E., Bottero, A.~G., Vinogradska, J., Berkenkamp, F., and Peters, J.
\newblock Model-based epistemic variance of values for risk-aware policy
  optimization, 2024{\natexlab{b}}.
\newblock URL \url{https://arxiv.org/abs/2312.04386}.

\bibitem[Mazumdar et~al.(2020)Mazumdar, Pacchiano, Ma, Jordan, and
  Bartlett]{mazumdar2020approximate}
Mazumdar, E., Pacchiano, A., Ma, Y., Jordan, M., and Bartlett, P.
\newblock On approximate thompson sampling with langevin algorithms.
\newblock In \emph{International Conference on Machine Learning}, pp.\
  6797--6807. PMLR, 2020.

\bibitem[Minka(2000)]{minka2000bayesian}
Minka, T.
\newblock Bayesian linear regression.
\newblock Technical report, Citeseer, 2000.

\bibitem[Moradipari et~al.(2023)Moradipari, Pedramfar, Zini, and
  Aggarwal]{moradipari2023improved}
Moradipari, A., Pedramfar, M., Zini, M.~S., and Aggarwal, V.
\newblock Improved bayesian regret bounds for thompson sampling in
  reinforcement learning.
\newblock In \emph{Thirty-seventh Conference on Neural Information Processing
  Systems}, 2023.
\newblock URL \url{https://openreview.net/forum?id=2EVTB1idyR}.

\bibitem[Mou et~al.(2024)Mou, Ho, Wainwright, Bartlett, and
  Jordan]{mou2019diffusion}
Mou, W., Ho, N., Wainwright, M., Bartlett, P., and Jordan, M.
\newblock A diffusion process perspective on posterior contraction rates for
  parameters.
\newblock \emph{SIAM Journal on Mathematics of Data Science}, 6\penalty0
  (2):\penalty0 553--577, 2024.

\bibitem[Nguyen-Tang et~al.(2024)Nguyen-Tang, Yin, Uehara, Wang, Wang, and
  Arora]{nguyen-tang2024posterior}
Nguyen-Tang, T., Yin, M., Uehara, M., Wang, Y.-X., Wang, M., and Arora, R.
\newblock Posterior sampling via langevin monte carlo for offline reinforcement
  learning, 2024.
\newblock URL \url{https://openreview.net/forum?id=WwCirclMvl}.

\bibitem[Nitanda et~al.(2022)Nitanda, Wu, and Suzuki]{pmlr-v151-nitanda22a}
Nitanda, A., Wu, D., and Suzuki, T.
\newblock Convex analysis of the mean field langevin dynamics.
\newblock In Camps-Valls, G., Ruiz, F. J.~R., and Valera, I. (eds.),
  \emph{Proceedings of The 25th International Conference on Artificial
  Intelligence and Statistics}, volume 151 of \emph{Proceedings of Machine
  Learning Research}, pp.\  9741--9757. PMLR, 28--30 Mar 2022.
\newblock URL \url{https://proceedings.mlr.press/v151/nitanda22a.html}.

\bibitem[O'~Donoghue(2021)]{NEURIPS2021_edb446b6}
O'~Donoghue, B.
\newblock Variational bayesian reinforcement learning with regret bounds.
\newblock In Ranzato, M., Beygelzimer, A., Dauphin, Y., Liang, P., and Vaughan,
  J.~W. (eds.), \emph{Advances in Neural Information Processing Systems},
  volume~34, pp.\  28208--28221. Curran Associates, Inc., 2021.
\newblock URL
  \url{https://proceedings.neurips.cc/paper_files/paper/2021/file/edb446b67d69adbfe9a21068982000c2-Paper.pdf}.

\bibitem[O'Donoghue et~al.(2018)O'Donoghue, Osband, Munos, and
  Mnih]{o2018uncertainty}
O'Donoghue, B., Osband, I., Munos, R., and Mnih, V.
\newblock The uncertainty bellman equation and exploration.
\newblock In \emph{International Conference on Machine Learning}, pp.\
  3836--3845, 2018.

\bibitem[Osband \& Van~Roy(2014)Osband and Van~Roy]{NIPS2014_1141938b}
Osband, I. and Van~Roy, B.
\newblock Model-based reinforcement learning and the eluder dimension.
\newblock In Ghahramani, Z., Welling, M., Cortes, C., Lawrence, N., and
  Weinberger, K. (eds.), \emph{Advances in Neural Information Processing
  Systems}, volume~27. Curran Associates, Inc., 2014.
\newblock URL
  \url{https://proceedings.neurips.cc/paper_files/paper/2014/file/1141938ba2c2b13f5505d7c424ebae5f-Paper.pdf}.

\bibitem[Osband \& Van~Roy(2017)Osband and Van~Roy]{osband2017posterior}
Osband, I. and Van~Roy, B.
\newblock Why is posterior sampling better than optimism for reinforcement
  learning?
\newblock In \emph{International conference on machine learning}, pp.\
  2701--2710. PMLR, 2017.

\bibitem[Osband et~al.(2013{\natexlab{a}})Osband, Russo, and
  Roy]{osband:thompson:nips:2013}
Osband, I., Russo, D., and Roy, B.~V.
\newblock (more) efficient reinforcement learning via posterior sampling.
\newblock In \emph{{NIPS}}, 2013{\natexlab{a}}.

\bibitem[Osband et~al.(2013{\natexlab{b}})Osband, Russo, and
  Van~Roy]{osband2013more}
Osband, I., Russo, D., and Van~Roy, B.
\newblock (more) efficient reinforcement learning via posterior sampling.
\newblock \emph{Advances in Neural Information Processing Systems}, 26,
  2013{\natexlab{b}}.

\bibitem[Osband et~al.(2023)Osband, Wen, Asghari, Dwaracherla, Ibrahimi, Lu,
  and Van~Roy]{pmlr-v216-osband23a}
Osband, I., Wen, Z., Asghari, S.~M., Dwaracherla, V., Ibrahimi, M., Lu, X., and
  Van~Roy, B.
\newblock Approximate {T}hompson sampling via epistemic neural networks.
\newblock In Evans, R.~J. and Shpitser, I. (eds.), \emph{Proceedings of the
  Thirty-Ninth Conference on Uncertainty in Artificial Intelligence}, volume
  216 of \emph{Proceedings of Machine Learning Research}, pp.\  1586--1595.
  PMLR, 31 Jul--04 Aug 2023.
\newblock URL \url{https://proceedings.mlr.press/v216/osband23a.html}.

\bibitem[Ouhamma et~al.(2022)Ouhamma, Basu, and Maillard]{ouhamma2022bilinear}
Ouhamma, R., Basu, D., and Maillard, O.-A.
\newblock Bilinear exponential family of mdps: Frequentist regret bound with
  tractable exploration and planning.
\newblock \emph{arXiv preprint arXiv:2210.02087}, 2022.

\bibitem[Phan et~al.(2019)Phan, Abbasi~Yadkori, and
  Domke]{NEURIPS2019_f3507289}
Phan, M., Abbasi~Yadkori, Y., and Domke, J.
\newblock Thompson sampling and approximate inference.
\newblock In Wallach, H., Larochelle, H., Beygelzimer, A., d\textquotesingle
  Alch\'{e}-Buc, F., Fox, E., and Garnett, R. (eds.), \emph{Advances in Neural
  Information Processing Systems}, volume~32. Curran Associates, Inc., 2019.
\newblock URL
  \url{https://proceedings.neurips.cc/paper/2019/file/f3507289cfdc8c9ae93f4098111a13f9-Paper.pdf}.

\bibitem[Pinneri et~al.(2020)Pinneri, Sawant, Blaes, Achterhold, Stueckler,
  Rolinek, and Martius]{PinneriEtAl2020:iCEM}
Pinneri, C., Sawant, S., Blaes, S., Achterhold, J., Stueckler, J., Rolinek, M.,
  and Martius, G.
\newblock Sample-efficient cross-entropy method for real-time planning.
\newblock In \emph{Conference on Robot Learning 2020}, 2020.
\newblock URL \url{https://corlconf.github.io/paper_217}.

\bibitem[Reisinger et~al.(2008)Reisinger, Stone, and
  Miikkulainen]{Reisinger2008}
Reisinger, J., Stone, P., and Miikkulainen, R.
\newblock Online kernel selection for bayesian reinforcement learning.
\newblock In \emph{International Conference on Machine Learning}, pp.\
  816--823, 2008.

\bibitem[Russo \& Van~Roy(2014)Russo and Van~Roy]{doi:10.1287/moor.2014.0650}
Russo, D. and Van~Roy, B.
\newblock Learning to optimize via posterior sampling.
\newblock \emph{Mathematics of Operations Research}, 39\penalty0 (4):\penalty0
  1221--1243, 2014.
\newblock \doi{10.1287/moor.2014.0650}.
\newblock URL \url{https://doi.org/10.1287/moor.2014.0650}.

\bibitem[Russo et~al.(2020)Russo, Roy, Kazerouni, Osband, and
  Wen]{russo2020tutorial}
Russo, D., Roy, B.~V., Kazerouni, A., Osband, I., and Wen, Z.
\newblock A tutorial on thompson sampling, 2020.

\bibitem[Sasso et~al.(2023)Sasso, Conserva, and Rauber]{pmlr-v202-sasso23a}
Sasso, R., Conserva, M., and Rauber, P.
\newblock Posterior sampling for deep reinforcement learning.
\newblock In Krause, A., Brunskill, E., Cho, K., Engelhardt, B., Sabato, S.,
  and Scarlett, J. (eds.), \emph{Proceedings of the 40th International
  Conference on Machine Learning}, volume 202 of \emph{Proceedings of Machine
  Learning Research}, pp.\  30042--30061. PMLR, 23--29 Jul 2023.
\newblock URL \url{https://proceedings.mlr.press/v202/sasso23a.html}.

\bibitem[Schlichting(2019)]{schlichting2019poincare}
Schlichting, A.
\newblock Poincar{\'e} and log--sobolev inequalities for mixtures.
\newblock \emph{Entropy}, 21\penalty0 (1):\penalty0 89, 2019.

\bibitem[Schulman et~al.(2017)Schulman, Wolski, Dhariwal, Radford, and
  Klimov]{schulman2017proximalpolicyoptimizationalgorithms}
Schulman, J., Wolski, F., Dhariwal, P., Radford, A., and Klimov, O.
\newblock Proximal policy optimization algorithms, 2017.
\newblock URL \url{https://arxiv.org/abs/1707.06347}.

\bibitem[Simchowitz et~al.(2021)Simchowitz, Tosh, Krishnamurthy, Hsu, Lykouris,
  Dudik, and Schapire]{simchowitz2021bayesian}
Simchowitz, M., Tosh, C., Krishnamurthy, A., Hsu, D.~J., Lykouris, T., Dudik,
  M., and Schapire, R.~E.
\newblock Bayesian decision-making under misspecified priors with applications
  to meta-learning.
\newblock \emph{Advances in Neural Information Processing Systems},
  34:\penalty0 26382--26394, 2021.

\bibitem[Steiner(2021)]{steinergao2021feynmankac}
Steiner, C.
\newblock A feynman-kac approach for logarithmic sobolev inequalities, 2021.

\bibitem[Stroock \& Zegarlinski(1992)Stroock and
  Zegarlinski]{stroock1992equivalence}
Stroock, D.~W. and Zegarlinski, B.
\newblock The equivalence of the logarithmic sobolev inequality and the
  dobrushin-shlosman mixing condition.
\newblock \emph{Communications in mathematical physics}, 144:\penalty0
  303--323, 1992.

\bibitem[Thompson(1933)]{thompson1933lou}
Thompson, W.
\newblock {On the Likelihood that One Unknown Probability Exceeds Another in
  View of the Evidence of two Samples}.
\newblock \emph{Biometrika}, 25\penalty0 (3-4):\penalty0 285--294, 1933.

\bibitem[Todorov et~al.(2012)Todorov, Erez, and Tassa]{todorov2012mujoco}
Todorov, E., Erez, T., and Tassa, Y.
\newblock Mujoco: A physics engine for model-based control.
\newblock In \emph{2012 IEEE/RSJ International Conference on Intelligent Robots
  and Systems}, pp.\  5026--5033. IEEE, 2012.
\newblock \doi{10.1109/IROS.2012.6386109}.

\bibitem[Towers et~al.(2024)Towers, Kwiatkowski, Terry, Balis, De~Cola, Deleu,
  Goul{\~a}o, Kallinteris, Krimmel, KG, et~al.]{towers2024gymnasium}
Towers, M., Kwiatkowski, A., Terry, J., Balis, J.~U., De~Cola, G., Deleu, T.,
  Goul{\~a}o, M., Kallinteris, A., Krimmel, M., KG, A., et~al.
\newblock Gymnasium: A standard interface for reinforcement learning
  environments.
\newblock \emph{arXiv preprint arXiv:2407.17032}, 2024.

\bibitem[Vempala \& Wibisono(2019)Vempala and Wibisono]{NEURIPS2019_65a99bb7}
Vempala, S. and Wibisono, A.
\newblock Rapid convergence of the unadjusted langevin algorithm: Isoperimetry
  suffices.
\newblock In Wallach, H., Larochelle, H., Beygelzimer, A., d\textquotesingle
  Alch\'{e}-Buc, F., Fox, E., and Garnett, R. (eds.), \emph{Advances in Neural
  Information Processing Systems}, volume~32. Curran Associates, Inc., 2019.
\newblock URL
  \url{https://proceedings.neurips.cc/paper/2019/file/65a99bb7a3115fdede20da98b08a370f-Paper.pdf}.

\bibitem[Wang et~al.(2023)Wang, Chen, and Murphy]{wang2023model}
Wang, C., Chen, Y., and Murphy, K.~P.
\newblock Model-based policy optimization under approximate bayesian inference.
\newblock In \emph{ICML Workshop on New Frontiers in Learning, Control, and
  Dynamical Systems}, 2023.

\bibitem[Wang(2001)]{wang2001logarithmic}
Wang, F.-Y.
\newblock Logarithmic sobolev inequalities: conditions and counterexamples.
\newblock \emph{Journal of Operator Theory}, pp.\  183--197, 2001.

\bibitem[Wang \& Ghosal(2023)Wang and Ghosal]{wang2023posterior}
Wang, K. and Ghosal, S.
\newblock Posterior contraction and testing for multivariate isotonic
  regression.
\newblock \emph{Electronic Journal of Statistics}, 17\penalty0 (1):\penalty0
  798--822, 2023.

\bibitem[Xu et~al.(2022)Xu, Zheng, Mazumdar, Azizzadenesheli, and
  Anandkumar]{pmlr-v162-xu22p}
Xu, P., Zheng, H., Mazumdar, E.~V., Azizzadenesheli, K., and Anandkumar, A.
\newblock {L}angevin {M}onte {C}arlo for contextual bandits.
\newblock In Chaudhuri, K., Jegelka, S., Song, L., Szepesvari, C., Niu, G., and
  Sabato, S. (eds.), \emph{Proceedings of the 39th International Conference on
  Machine Learning}, volume 162 of \emph{Proceedings of Machine Learning
  Research}, pp.\  24830--24850. PMLR, 17--23 Jul 2022.
\newblock URL \url{https://proceedings.mlr.press/v162/xu22p.html}.

\bibitem[Yamamoto et~al.(2024)Yamamoto, Oko, Yang, and
  Suzuki]{yamamoto2023mean}
Yamamoto, K., Oko, K., Yang, Z., and Suzuki, T.
\newblock Mean field langevin actor-critic: Faster convergence and global
  optimality beyond lazy learning.
\newblock In \emph{Forty-first International Conference on Machine Learning},
  2024.

\bibitem[Zheng et~al.(2024)Zheng, Deng, Moya, and Lin]{pmlr-v238-zheng24b}
Zheng, H., Deng, W., Moya, C., and Lin, G.
\newblock Accelerating approximate {T}hompson sampling with underdamped
  {L}angevin {M}onte {C}arlo.
\newblock In Dasgupta, S., Mandt, S., and Li, Y. (eds.), \emph{Proceedings of
  The 27th International Conference on Artificial Intelligence and Statistics},
  volume 238 of \emph{Proceedings of Machine Learning Research}, pp.\
  2611--2619. PMLR, 02--04 May 2024.
\newblock URL \url{https://proceedings.mlr.press/v238/zheng24b.html}.

\bibitem[Zintgraf et~al.(2021)Zintgraf, Schulze, Lu, Feng, Igl, Shiarlis, Gal,
  Hofmann, and Whiteson]{zintgraf2021varibad}
Zintgraf, L., Schulze, S., Lu, C., Feng, L., Igl, M., Shiarlis, K., Gal, Y.,
  Hofmann, K., and Whiteson, S.
\newblock Varibad: variational bayes-adaptive deep rl via meta-learning.
\newblock \emph{The Journal of Machine Learning Research}, 22\penalty0
  (1):\penalty0 13198--13236, 2021.

\end{thebibliography}

\clearpage

\onecolumn
\appendix
    \section{Notation}
        \begin{table}[h!]
        \caption{Table of notations.}
        \centering 
        \begin{tabular}{ll}
            \toprule
            $a_{\epindex,h}$ & Action in timestep $h$ of episode $\epindex$ \\
            $s_{\epindex,h}$ & State in timestep $h$ of episode $\epindex$. \\
             $\epindex$ & Episode index. \\
             $ \Epindex$ & Total number of episodes. \\
             $h$ & Current step in episode. \\
             $H$  & Horizon, amount of steps in an episode. \\
           $\totalT$ & 
           Total amount of agent interactions, $\totalT=\Epindex H$\\
           $\dparam$ & Dimension of the parameter space of $\M$ \\
           d & Dimension of the state space \\
             $n$ & Total amount of available data. Usually $n=lH$  at the start of episode $l+1$. \\
            $\datat$ & The states, actions and transitions observed until start of episode $
            \epindex$. \\
            $\M$& An MDP and its parametrization in $\R^\dparam$. \\
             $\Me$ & Sampled MDP in episode $\epindex$.\\
             $\pol$ & Policy \\
             $\pol^*$ & Optimal policy for the true MDP \\
             $\pol^*(\M)$ & Optimal policy for MDP $\M$. \\
             $V_M^{\pol}$  & Value of policy $\pi$ in MDP $M$. \\
            $\BRT$ & Bayesian regret, \Cref{eq:bayes_regret} \\
            $\Delta_{\max}$ & Maximal possible regret for an episode, $\max_{\pol} V_{\pi,1}^{\M}(s_1)-\min_{\pol} V_{\pi,1}^{\M}(s_1)$\\
             $B_R$ & Upper bound on absolute value of average reward in \Cref{thm:subgregret_lsi}. \\
             $\lsi$ & Log-Sobolev constant.\\
             $\lsipep$ & Log-Sobolev constant for the average transitions in episode $\epindex$.\\
             $\lsirep$ & Log-Sobolev constant for the average rewards in episode $\epindex$.\\
             $\alpha_{\theta^*}$ & Log-Sobolev constant for the data likelihood $\lik(X\mid \theta^*)$.\\
             $\subgM$ & Log-Sobolev constant for the likelihood of the true MDP $\lik(X\mid \M)$.\\             $L$, &L-smooth constant of the log likelihood function. \\
             $\gamma$ & Temperature parameter in Langevin dynamics, $\gamma=n$ in Bayesian posterior setting. \\
             $\epslange$ & Langevin  sampling error \\
             $L_\M$ & Mean-Lipschitz parameter from \Cref{ass:lip} \\
            $L_{\Bar{R}},L_{\Bar{\transition}}$& Lipschitz parameters for mean transition and reward.\\
             $\prior(\cdot), \bel(\cdot | x)$ & Prior, posterior and approximate posterior on the parameters.\\
             $\hat{\bel}(\cdot|x)$ &  Approximate posterior on the parameters from the Langevin sampling.\\
             $\lik(x\mid \theta)$ & Likelihood of data $x$ in model $\theta$.\\
              $\E$ & Expectation \\
            $\KL{P}{Q}$ &  Kullback–Leibler divergence between $P$ and $Q$.\\
             $\omega(\ball)$ & Lower bound on prior mass from \Cref{ass:mass}.\\
             $g(H,\mathcal{S},\mathcal{A})$ & The non $\totalT$ dependent part of the order of regret. \\
             $\theta$& A  set of parameters, in the case of MDPs this is written as $\M$.\\
              \bottomrule
        \end{tabular}
        \label{tab:Notation}
    \end{table}
\newpage

\section{Algorithmic Details: SARAH-LD}
    For completeness, we describe the pseudocode of SARAH-LD~\citep{kinoshita2022improved}  algorithms in \Cref{alg:sarahld} as well as a theorem on the gradient complexity of SARAH-LD in \Cref{thm:sarahldcomplex}. We slightly modify SARAH-LD notation to Bayesian posterior setting for usage in LaPSRL.

    \subsection{SARAH-LD}
    SARAH-LD, found in \Cref{alg:sarahld}, is a Langevin algorithm by \citet{kinoshita2022improved} that utilizes variance reduction techniques for reduced gradient complexity.
    In each epoch, it performs a gradient step on the full dataset, before taking smaller mini-batches for efficiency. In each parameter update, noise is added to the gradient in line with the Langevin dynamics. In the inner loop of mini batches, a difference of batches is done to reduce the variance of the estimate. At the end of the algorithm, it returns a sample, which is bounded in distance from its target distribution. This can be seen more formally in the following theorem. 
    \begin{algorithm}
    \caption{SARAH-LD}\label{alg:sarahld}
    \begin{algorithmic}[1]
    \STATE \textbf{Input}:  step size $\eta>0$, batch size $B$, epoch length $m$, data X, likelihood $f$, prior $\bel(\theta)$, initial sample $\rho_0$.
    \STATE \textbf{Initialization}: $\theta_0=\rho_0, \theta^{(0)}=\theta_0$
    \FOR{$s=0,1, \ldots,(K / m)$}
    \STATE $v_{s m}=\nabla \lik\left(X|\theta^{(s)}\right)$
    \STATE randomly draw $\epsilon \sim N\left(0, I_{d \times d}\right)$
    \STATE $\theta_{s m+1}=\theta_{s m}-\eta v_{s m}+\sqrt{2 \eta / n} \epsilon$
        \FOR{$l=1, \ldots, m-1$}
        \STATE $k=s m+l$
        \STATE randomly pick a subset $I_k$ from $\{1, \ldots, n\}$ of size $\left|I_k\right|=B$
        \STATE randomly draw $\epsilon \sim N\left(0, I_{d \times d}\right)$
        \STATE $v_k=\frac{1}{B} \sum_{i_k \in I_k}\left(\nabla (\lik\left(X_{i_k}\mid\theta_k\right)+1/n \bel(\theta_k))-\nabla (\lik\left(X_{i_k}\mid\theta_{k-1}\right)+1/n \bel(\theta_{k-1}))\right)+v_{k-1}$
        \STATE $\theta_{k+1}=\theta_k-\eta v_k+\sqrt{2 \eta / n} \epsilon$
        \ENDFOR
    \STATE $\theta^{(s+1)}=\theta_{(s+1) m}$
    \ENDFOR
    \end{algorithmic}
    \end{algorithm}
        \begin{theorem}[Corollary 2.1 of~\citep{kinoshita2022improved}]
        \label{thm:sarahldcomplex}
		Under \Cref{ass:smooth,ass:lsi}, for all $\epsilon \geq 0$, if we choose step size $\eta$ such that $\eta \leq \frac{3 \alpha \epsilon}{48 \gamma \lsi^2}$, then a precision $\KL{\rho_k}{\nu} \leq \epsilon$ is reached after $k \geq \frac{\gamma}{\lsi \eta} \log \frac{2 \KL{\rho_0}{\nu}}{\epsilon}$ steps of SARAH-LD. Especially, if we take $B=m=\sqrt{n}$ and the largest permissible step size $\eta=\min(\frac{\lsi}{16 \sqrt{2} L^2 \sqrt{n} \gamma}, \frac{3 \lsi \epsilon}{320 d L^2 \gamma})$, then the gradient complexity becomes
		$\Otilde\left(\left(n+\frac{d n^{\frac{1}{2}}}{\epsilon}\right) \cdot \frac{\gamma^2 L^2}{\lsi^2}\right).$
		\label{thm:langevin_complexity}
	\end{theorem}

\section{Regret Bounds for PSRL with Exact Posteriors}
\subsection{Confidence Intervals for Isoperimetric Data Distributions} 

\subgconf*
\begin{proof}\label{proof:subgconf} 
Here we will use notation $x^{(n)}=\{x_i\}_0^{n-1}$, to refer to the set of data of size $n$. 

\textit{Step 1: From log-Sobolev to log-likelihood concentration.}
If the data distribution $\lik(x \mid \theta^*)$ is $\subgc$ log-Sobolev, we have from property of sub-Gaussian concentration of Lipschitz functions, as seen in \Cref{Eq:conc}, 
    \begin{align}
        \Pr(|\log \lik(x\mid \theta)-\E_{x \sim \lik(x\mid \theta^*)}\log \lik(x\mid \theta)|\ge t) &\le 2e^{-\frac{t^2\subgc}{L_x^2}} \\
        \implies |\log \lik(x \mid \theta) - \E_{x\sim \lik(x\mid \theta^*)} \log \lik(x\mid \theta)| &\le \sqrt{\frac{L_x^2}{\subgc}\ln\left(\frac{2}{\delta}\right)}\,,
    \end{align}
    with probability at least $1-\delta$. 
    
    Due to Hoeffding type bound on sum of conditionally independent sub-Gaussians~\citep{jin2019shortnoteconcentrationinequalities}, we get
    \begin{align}
        |\log \lik(x^{(n)} \mid \theta) - \E_{x^{(n)}\sim \lik(x^{(n)} \mid \theta^*)} \log \lik(x^{(n)}\mid \theta)| &\le {\sqrt{n\frac{L_x^2}{\subgc}\ln\left(\frac{2}{\delta}\right)}}\,, \label{eq:event}
    \end{align}
    with probability at least $1-\delta$.%
    The final implication comes from the sum of $n$ random variables that are $\frac{L_x^2}{\subgc}$ sub-Gaussians being $n\frac{L_x^2}{\subgc}$ sub-Gaussian. We denote the event from \Cref{eq:event} as $E_{\Bar{x}}$.
    
\textit{Step 2: From log-likelihood concentration to bounding probabilities. }
Note that $\E_{x\sim \lik(x\mid \theta^*)} \log \lik(x\mid \theta)=-\He[\lik(x\mid \theta^*), \lik(x\mid \theta)]$ where $\He(\cdot,\cdot)$ is the cross entropy. If $\theta=\theta^*$ this instead becomes the regular entropy $\He[\lik(x\mid \theta^*)]$.   
This implies that
\begin{align}%
    \lik(x^{(n)} \mid \theta)& \geq e^{\E_{x^{(n)}\sim \lik(x^{(n)} \mid \theta^*)} \log \lik(x^{(n)}\mid \theta)- \sqrt{n\frac{L_x^2}{\subgc}\ln\left(\frac{2}{\delta}\right)}} = e^{- \He(\lik(x^{(n)}\mid\theta^*), \lik(x^{(n)} \mid\theta))- \sqrt{n\frac{L_x^2}{\subgc}\ln\left(\frac{2}{\delta}\right)}} \label{eq:lb_likelihood}\\
     \lik(x^{(n)} \mid \theta) &\leq e^{\E_{x^{(n)}\sim \lik(x^{(n)} \mid \theta^*)} \log \lik(x^{(n)}\mid \theta)+ \sqrt{n\frac{L_x^2}{\subgc}\ln\left(\frac{2}{\delta}\right)}} = e^{- \He(\lik(x^{(n)}\mid\theta^*), \lik(x^{(n)}\mid\theta))+ \sqrt{n\frac{L_x^2}{\subgc}\ln\left(\frac{2}{\delta}\right)}}\,.\label{eq:ub_likelihood}
\end{align}

\textit{Step 3: From bounding probabilities to concentration of posteriors.}
   \begin{align}
        \bel\left( \Theta \setminus \ball \mid x^{(n)}, E_{\Bar{x}}\right) = & 
        \frac{\int_{\Theta \setminus \ball} \lik\left(x^{(n)} \mid \theta, E_{\Bar{x}} \right) d\bel(\theta)}
        {\int_\Theta \lik\left(x^{(n)} \mid \theta,E_{\Bar{x}} \right) d\bel(\theta)} \\
        \le & \frac{\int_{\Theta \setminus \ball} \lik\left(x^{(n)} \mid \theta,E_{\Bar{x}} \right) d\bel(\theta)}
        {\int_{\ball} \lik\left(x^{(n)} \mid \theta,E_{\Bar{x}} \right) d\bel(\theta)} \\
        \le & \frac{\int_{\Theta \setminus \ball} e^{- \He({\lik(x^{(n)}|\theta^*)},{\lik(x^{(n)}|\theta)})+ \sqrt{n\frac{L_x^2}{\subgc}\ln\left(\frac{2}{\delta}\right)}} d\bel(\theta)}
        {\int_{\ball} e^{- \He({\lik(x^{(n)}|\theta^*)},{\lik(x^{(n)}|\theta)})- \sqrt{n\frac{L_x^2}{\subgc}\ln\left(\frac{2}{\delta}\right)}} d\bel(\theta)}\,.
    \end{align}
   The last inequality holds due to Equation~\eqref{eq:lb_likelihood} and \eqref{eq:ub_likelihood}. %
   
Now, we proceed as follows
\begin{align}
        \bel\left( \Theta \setminus \ball \mid x^{(n)}, E_{\Bar{x}}\right)  \le & e^{ 2\sqrt{n\frac{L_x^2}{\subgc}\ln\left(\frac{2}{\delta}\right)}}\frac{\int_{\Theta \setminus \ball} e^{- \He({\lik(x^{(n)}|\theta^*)},{\lik(x^{(n)}|\theta)})} d\bel(\theta)}
        {\int_{\ball} e^{- \He({\lik(x^{(n)}|\theta^*)},{\lik(x^{(n)}|\theta)})} d\bel(\theta)}        \\
        = & e^{ 2\sqrt{n\frac{L_x^2}{\subgc}\ln\left(\frac{2}{\delta}\right)}}\frac{\int_{\Theta \setminus \ball} e^{- \KL{\lik(x^{(n)}\mid \theta^*)}{ \lik(x^{(n)} \mid \theta)}} d\bel(\theta)}
        {\int_{\ball} e^{-\KL{\lik(x^{(n)} \mid \theta^*)}{ \lik(x^{(n)} \mid \theta)}} d\bel(\theta)} \label{eq:entropy} \\
        \le & e^{ 2\sqrt{n\frac{L_x^2}{\subgc}\ln\left(\frac{2}{\delta}\right)}}\frac{e^{- \inf_{\theta \not\in 
        \ball} \KL{\lik(x^{(n)}\mid \theta^*)}{ \lik(x^{(n)}\mid \theta)}} \int_{\Theta \setminus \ball} d\bel(\theta)}
        {\int_{\ball} e^{- n L_{\theta}\|{\theta^*}-{ \theta}\|} d\bel(\theta)} \label{eq:lip}\\
        \le & \frac{1}{\bel(\ball)\omega(\ball)^n}e^{-\inf_{\theta \notin 
        \ball} \KL{\lik(x^{(n)}|\theta^*)}{\lik(x^{(n)}|\theta)}+2\sqrt{n\frac{L_x^2}{\subgc}\ln\left(\frac{2}{\delta}\right)}} \label{eq:mass} 
\end{align}
    \Cref{eq:entropy} is from multiplying top and bottom with $e^{\He[\lik(x\mid \theta^*)]}$ and the definition of KL-divergence.
    \Cref{eq:lip} comes from the Lipschitz property of the log-likelihood $\lik(x^{(n)}\mid\theta)$ and therefore also the KL-divergence.
    \Cref{eq:mass} comes from \Cref{ass:mass} and Jensen's inequality using $\phi(y)=y^n$ which is convex for non-negative values of $y$. 
    Specifically, by setting $y=e^{- L_{\theta}\|{\theta^*}-{ \theta}\|}$ for $\theta\in \ball$ and considering convexity and positivity of $y^n$ for any closed $\ball$ around $\theta^*$, we obtain through Jensen's inequality that %
\[
\int_\ball \phi(\theta) d\bel(\theta) = 
\E[\phi | \theta \in \ball]  \bel(\ball)
=
\bel(\ball) \int_\ball \phi(\theta) d\bel(\theta | \theta \in \ball) 
\geq
\bel(\ball) \phi\left(\int_\Theta f(\theta) d\bel(\theta | \ball) \right)\,.
\]
Note that for conditional expectations
    $\E[f(x) | x \in S] = \int_{x \in S} f(x) P(x)/ P(S)$.
      This completes the step together with bounding the probability by one.

    The assumption in \Cref{ass:mass} is reasonable considering that the exponent is zero around $\theta^*$ and the prior has a minimum mass around it.

\textit{Step 4. Constructing the confidence interval.} Now, if we upper bound the probability $ \bel\left( \Theta \setminus \ball \mid x^{(n)}, E_{\Bar{x}}\right)$ with $\delta'\in(0,1)$, we get
\begin{align*}
      \inf_{\theta \notin 
        \ball} \KL{\lik(x^{(n)}|\theta^*)}{\lik(x^{(n)}|\theta)}-2\sqrt{n\frac{L_x^2}{\subgc}\ln\left(\frac{2}{\delta}\right)}) &\geq \ln \left(\frac{1}{\delta'(1-\delta') \omega^n(\ball)}\right)\\
        \implies \inf_{\theta \notin 
        \ball} \KL{\lik(x^{(n)}|\theta^*)}{\lik(x^{(n)}|\theta)} \geq 2\sqrt{n\frac{L_x^2}{\subgc}\ln\left(\frac{2}{\delta}\right)}) &+ \ln\left(\frac{1}{\delta'}\right)\\
        \implies \inf_{\theta \notin 
        \ball} \KL{\lik(x|\theta^*)}{\lik(x|\theta)} \geq 2\sqrt{\frac{L_x^2}{n\subgc}\ln\left(\frac{2}{\delta}\right)}) &+ \frac{1}{n}\ln\left(\frac{1}{\delta'}\right)\,.
\end{align*}
    The first implication is true due to the observations that $1-\delta' < 1$ and $\omega(\ball) <1$. The second implication is true due to tensorisation  property of KL-divergence.
    
If $\ball$ is defined as
    \begin{align}
        \ball =\Bigg\{ \theta~\mid ~\KL{\lik(x|\theta^*)}{\lik(x|\theta)}\leq 2\sqrt{\frac{L_x^2}{n\subgc}\ln\left(\frac{2}{\delta}\right)} + \frac{
        1
        }{n}\ln\left(\frac{1}{\delta'}\right)\Bigg\}\,.
    \end{align}
    Then, we get
    \begin{align}
         \bel\left( \theta \not\in \ball \mid x^{(n)}, E_{\Bar{x}}\right)  \leq \delta'\,~~\implies~~\bel\left( \theta \not\in \ball \mid x^{(n)}\right)  \leq \delta'+\delta.
    \end{align}
    Setting $\delta'=\delta$ completes our proof.
\end{proof}

\subgregret*
\begin{proof}\,
\label{subgregret_proof}
\textbf{Step 1: Bayesian Regret in PSRL Scheme.} If we consider the first step of an episode $l$, the total number of completed steps is $n=(l-1)H$. In PSRL, we sample $\Me \sim \Pr\left(\M \in \cdot \mid \mathcal{H}_n\right)$, where the sampling is from the posterior distribution of $\M$. \cite{osband2013more} observes that for any $\sigma(\datat)$ measurable function $f$, given the history of transitions $\datat \equiv \mathcal{H}_{(\epindex-1) H}=\left\{\left(s_{t',h}, a_{\epindex',h}, s_{\epindex', h+1}\right)_{\epindex'<\epindex, h \leq H}\right\}$, we have $\E[f(\Me)]=\E[f(\M)]$. This family of $f$'s includes the value function. 
Therefore, we have $\mathbb{E}\left[V_{\pol_\epindex, 1}^{\Me}\left(s_{\epindex,1}\right)\right]=\mathbb{E}\left[V_{\pol^*, 1}^{\M}\left(s_{\epindex,1}\right)\right]$. Hence, the Bayes regret (\Cref{eq:bayes_regret}) of PSRL can be re-written as
\begin{align*}
    \BRT = \E\left[ \sum_{\epindex=1}^\Epindex V_{\pol_l,1}^{\Me}(s_{\epindex,1})-V_{\pol_\epindex,1}^{\M}(s_{\epindex,1})\right]\,.
\end{align*}

\textbf{Step 2. Recursion with Bellman Equation.} \cite{chowdhury2019online} further shows that by a recursive application of the Bellman equation, we can decompose this regret into the expectation of a martingale difference sequence, and the difference of the next step value functions in the sampled and true MDPs. 
Specifically, 
\begin{align*}
    \BRT= \E\left[ \sum_{\epindex=1}^\Epindex\sum_{h=1}^H \mathcal{T}^{\pol_l}_{\Me,h}(V_{\pol_l,h+1}^{\Me})(s_{\epindex,h})-\mathcal{T}^{\pol_l}_{\M,h}(V_{\pol_\epindex,h+1}^{\M})(s_{\epindex,h}) +  \sum_{\epindex=1}^\Epindex\sum_{h=1}^H m_{{\epindex},h}\right]\,.
\end{align*} 
Here, $\mathcal{T}^{\pol}_{\M,h}$ denotes the Bellman operator at step $h$ of the episode due to a policy $\pol$ and MDP $\M$, and is defined as $\mathcal{T}^{\pol}_{\M,h}(V_{\pol,h+1}^{\M})(s_{\epindex,h}) =R(s,\pol(s,h)) +\E_{s,\pol(s,h)}[V|\M]$. 
In addition, $m_{\epindex,h}=\E_{s_{\epindex,h}, a_{{\epindex},h}}^{\M}\left[V_{\pol_{{\epindex}}, h+1}^{{\M}_{n}}\left(s_{\epindex, h+1}\right)-V_{\pol_{\epindex}, h+1}^{\M}\left(s_{{\epindex}, h+1}\right)\right]-\left(V_{\pol_{\epindex}, h+1}^{{\M}_n}\left(s_{\epindex, h+1}\right)-V_{\pol_{\epindex}, h+1}^{\M}\left(s_{\epindex, {h+1}}\right)\right)$ is a martingale difference sequence satisfying $\mathbb{E}\left[m_{\epindex,h}\right]=0$. 

\textbf{Step 3. From Value Function to KL.} Now, from~\citep{boucheron2003concentration}, we obtain the transportation inequalities stating that
\begin{align}
    \label{eq:transport}
    \E_P[x] - \E_Q[x] \leq \sqrt{2\mathbb{V}_P(x) \KL{Q}{P}} \,.
\end{align}

Then an application of the transportation inequality yields
$$
\BRT \leq H \mathbb{E}\left[\sum_{\epindex=1}^\Epindex \sum_{h=1}^H\sqrt{2 \sigma_R^2 \mathrm{KL}_{s_h^\epindex, a_h^\epindex}\left(\M\| {\M}_n\right)}\right],
$$%
Here, $\sigma_R^2$ is the maximum variance of rewards at each step.

Finally, using our concentration bounds on $\KL{\M}{\M}$ under posterior distributions and then Cauchy-Schawrtz inequality yields $$
BR(T| \mathcal{E}^{\star}) \leq H \sigma_R \sqrt{2T \sum_{\epindex=1}^\Epindex \sum_{h=1}^H\mathrm{KL}_{s_h^\epindex, a_h^\epindex}\left(\M\| {\M}_n\right)} \leq H \sigma_R \sqrt{2T \left(T \frac{L_x^2}{\subgM}\ln\left(\frac{2}{\delta}\right)\right)^{1/2}+ 2T\ln\left(\frac{T}{\delta}\right)}.$$%
Here, $\mathcal{E}^{\star}$ denotes the event of the distribution of data concentrating in the set $\ball_n$ around $\M$ under the $n$-th step posterior for any $n\geq 1$.

\textbf{Step 4. Putting the Events Together.} Due to bounded mean of the rewards, we can always bound $\BRT \leq T B_{\mathcal{R}}$.

Thus, we have
$$
\BRT=\mathbb{E}\left[\mathcal{R}(T) \mathbb{I}_{\mathcal{E}^{\star}}+\mathcal{R}(T) \mathbb{I}_{\left(\mathcal{E}^{\star}\right)^c}\right] \leq BR(T| \mathcal{E}^{\star}) + 2 T B_{{R}}\left(1-\mathbb{P}\left(\mathcal{E}^{\star}\right)\right)
$$
From Theorem~\ref{lemma:subgconf}, $\mathbb{P}\left(\mathcal{E}^{\star}\right) \geq 1-2\delta$. This implies for any $\delta \in(0,1)$ that the Bayes regret
\begin{align}
\BRT &\leq H \sigma_R \sqrt{2T\left(T \frac{L_x^2}{\subgM}\ln\left(\frac{2}{\delta}\right)\right)^{1/2} + 2T \ln\left(\frac{T}{\delta}\right)}+ 4T B_{{R}} \delta \notag\\ &\leq 2B_{{R}} + 2 H \sigma_R\sqrt{T \ln\left(2T\right)} + \sqrt{2} H \sigma_R T^{3/4} \left(\frac{L_x^2}{{\subgM}}\right)^{1/4}\log^{1/4}\left(4T\right) \,.
\end{align}
The proof is completed by setting $\delta=\frac{1}{2T}$.
\end{proof}

\newpage
\subsection{Regret for Posteriors with Linear LSI Constants}

To study the alternate approach we first need to define the one step future value function $U(\varphi) $ as the expected value of the optimal policy $\pole$ in $\Me$ where  $\varphi$ is the distribution of next state samples. This gives $U^\Me_h(\varphi) = \E_{s'\sim \varphi} \left [ V^{\pole, h+1}(s')\right ].$
We use this definition, which is also used in previous work, \citep{chowdhury2019online, NIPS2014_1141938b}, to make a Lipschitz assumption on the next step value function $U$ with respect to the means of the distributions. 

\begin{assumption}[One step value function is Lipschitz in the mean]
    \label{ass:lip}
    For any $\varphi_1, \varphi_2$ distributions over $\mathcal{S}$ with $1\le h\le H$,
    \begin{equation}
        |U^\M_h(\varphi_1) - U^\M_h(\varphi_2)| \le L_\M \norm{\overline{\varphi_1}-\overline{\varphi_2}}_2
    \end{equation}
    where $\overline{\varphi_1}$ and $\overline{\varphi_2}$ are the means of the respective distributions.
\end{assumption}

\subgregretlsi*
\begin{proof}
    \label{proof:subgregret_lsi}
    This proof follows the general flow from \citet{chowdhury2019online} for Kernel PSRL but with totally different confidence bounds. 

    For PSRL, we have $\pol_\epindex=\argmax_\pol V_{\pol,1}^{\Me}$. We also denote the optimal policy  for the true MDP $\M$ as  $\pol_*=V_{\pol,1}^{\M}$.
    With the observation that under any observed history $\mathcal{H}_{\epindex-1}$ we have $\E[V_{\pol_\epindex,1}^{\Me}(s_{\epindex,1})\mid \mathcal{H}_{\epindex-1}]=\E[V_{\pol_*,1}^{\M}(s_{\epindex,1})\mid\mathcal{H}_{\epindex-1}]$, since they are both sampled from the same distribution.
    Marginalising we obtain:
        \begin{align}
       \E[\text{Regret}(\totalT)] & \defn  \sum_{\epindex=1}^\Epindex \E\left[V_{\pol_*,1}^{\M}(s_{\epindex,1})-V_{\pol_\epindex,1}^{\M}(s_{\epindex,1})\right] \\ &= \sum_{\epindex=1}^\Epindex \E\left[V_{\pol_*,1}^{\M}(s_{\epindex,1})-V_{\pol_\epindex,1}^{\Me}(s_{\epindex,1})\right] + \E\left[V_{\pol_\epindex,1}^{\Me}(s_{\epindex,1})-V_{\pol_\epindex,1}^{\M}(s_{\epindex,1})\right] \\
            &=\sum_{\epindex=1}^\Epindex \E\left[V_{\pol_\epindex,1}^{\Me}(s_{\epindex,1})-V_{\pol_\epindex,1}^{\M}(s_{\epindex,1})\right]
        \end{align}

        Next, we use Lemma 7 and observation after eq 50 from   \citep{chowdhury2019online}
        and obtain 
        \begin{align}
                    \E[\text{Regret}(\totalT)]  \le \E\left[\sum_{\epindex=1}^\Epindex \sum_{h=1}^H  \left[|\Bar{R}_{\Me}(s_{\epindex,h},a_{\epindex,h})-\Bar{R}_*(s_{\epindex,h},a_{\epindex,h})|+ L_\Me ||\Bar{\transition}_{\Me}(s_{\epindex,h},a_{\epindex,h}) - \Bar{\transition}_*(s_{\epindex,h},a_{\epindex,h})||_2\right]\right].
        \end{align}

 where $\Bar{\transition}_\M$ and $\Bar{R}_\M$ are the mean of the transition  and reward distributions for MDP $\M$. 

        We define two confidence sets 
        \begin{align}
            C_{R,\epindex,h}& = \left\{|\Bar{R}_\M(s_{\epindex,h})-E_{\bel(\M\mid \datat)}[\Bar{R}_\M(s_{\epindex,h})]|\le \sqrt{\frac{L_{\Bar{R}}^2 \log{1/\delta}}{\lsirep}} \right\} \\
            C_{\transition,\epindex,h}& = \left\{\|\Bar{\transition}_\M(s_{\epindex,h},a_{\epindex,h})-E_{\bel(\M\mid \datat)}[\Bar{\transition}_\M(s_{\epindex,h},a_{\epindex,h})]\|_2\le \sqrt{\frac{d L_{\Bar{\transition}}^2 \log{1/\delta}}{\lsipep}}\right\}
        \end{align}
        
        Define events $E_*\defn\{\Bar{R}_* \in C_{R,\epindex,h}, \Bar{\transition}_* \in C_{\transition,\epindex,h}, \forall 1 \le \epindex\le \Epindex, 1\le h\le H\}$ and $E_\Me\defn\{\Bar{R}_\Me \in C_{R,\epindex,h}, \Bar{\transition}_\Me \in C_{\transition,\epindex,h}, \forall 1\le \epindex\le \Epindex, 1\le h\le H\}$. Now We fix $0 \le \delta \le 1$ and from property on sub-Gaussian concentration for log-Sobolev posteriors in \Cref{Eq:conc}, we get $\Pr(E_\M)=\Pr(E_*)=1-2H\Epindex\delta$. Taking the union of these events $E\defn E_\M \cap E_*$ with $\Pr(E^c) \le \Pr(E_\M^c) + \Pr(E_*^c) \le 4\Epindex H\delta$. 
        We also have that $\E[L_\Me]=\E[L_\M]$ such that $\E[L_\Me|E] \le \frac{\E[L_\M]}{P(E)} \le \frac{\E[L_\M]}{1-4\Epindex H \delta}$.

        Combining the results we then get an upper bound on the Bayesian regret
        \begin{align}
           \E[\sum_{\epindex=1}^\Epindex \sum_{h=1}^H  \left[|\Bar{R}_{\Me}(s_{\epindex,h},a_{\epindex,h}))-\Bar{R}_*(s_{\epindex,h},a_{\epindex,h})|\mid E \right]+ \E \left[L_\Me ||\Bar{\transition}_{\Me}(s_{\epindex,h},a_{\epindex,h}) - \Bar{\transition}_*(s_{\epindex,h},a_{\epindex,h})||_2\right | E]  + 2B_R4\Epindex H\delta \\
           \le 2H\left( L_{\Bar{R}} \sqrt{ \log{1/\delta}}
           \sum_{\epindex=1}^\Epindex \frac{1}
    {\sqrt{\lsirep}} + \frac{\E[L_\M]}{1-2\Epindex H \delta} L_{\Bar{\transition}}\sqrt{d  \log{1/\delta}}
    \sum_{\epindex=1}^\Epindex \frac{1}{\sqrt{\lsipep}}\right ) + 8B_R\Epindex H\delta
        \end{align}

        Setting $\delta=\frac{1}{8\Epindex H}$
we obtain 
\begin{align}
\E[\text{Regret}(\totalT)]\le 2H\left( L_{\Bar{R}} \sqrt{ \log{8T}}
           \sum_{\epindex=1}^\Epindex \frac{1}
    {\sqrt{\lsirep}} + 2\E[L_\M] L_{\Bar{\transition}}\sqrt{d  \log{8T}}
    \sum_{\epindex=1}^\Epindex \frac{1}{\sqrt{\lsipep}}\right ) + B_R
\end{align}
\end{proof}

\section{Regret Bounds and Sample Complexity for LaPSRL with Approximate Posteriors}
\approxregret*

	\begin{proof}
    \label{proof:approxregret}

 Let $\Me,\M_\epindex \sim \Pr(\M), \, \M'_\epindex \sim \Q(\M)$. The policy $\pol_\epindex$ is the optimal policy corresponding to $\Me$ and $\pol'_\epindex$ the policy corresponding to $\M'_\epindex$. 
 \begin{align}
   \E_{\bel_\epindex,\Q_\epindex}[V^{\M}_{\pi^*} - V^{\M}_{\pi'_\epindex}]
     & = \E_{\bel_\epindex,\Q_\epindex}[V^{\M}_{\pi^*} - V^{\M'_\epindex}_{\pi'_\epindex} + V^{\M'_\epindex}_{\pi'_\epindex} -V^{\M}_{\pi'_\epindex}] \\
     & = \E_{\bel_\epindex,\Q_\epindex}[V^{\M}_{\pi^*} - V^{\Me}_{\pi_\epindex} + V^{\Me}_{\pi_\epindex}  - V^{\M'_\epindex}_{\pi'_\epindex} + V^{\M'_\epindex}_{\pi'_\epindex}  -V^{\M}_{\pi'_\epindex}] \\
     & = \E_{\bel_\epindex,\Q_\epindex}[[V^{\M}_{\pi^*} - V^{\Me}_{\pi_\epindex}] + [V^{\Me}_{\pi_\epindex}  - V^{\M'_\epindex}_{\pi'_\epindex}] + [V^{\M'_\epindex}_{\pi'_\epindex}  -V^{\M}_{\pi'_\epindex}]] \\
        &\le \E_{\bel_\epindex}[V^{\M}_{\pi^*} - V^{\Me}_{\pi_\epindex}] +  \Delta_{\text{max}}\sqrt{\frac{\epslange }{2}}  + \Delta_{\text{max}}\sqrt{\frac{\epslange }{2}}] \\
     &= \E_{\bel_\epindex}[V^{\M}_{\pi^*} - V^{\M}_{\pi_\epindex}] +   \sqrt{2}\Delta_{\text{max}}\sqrt{\epslange}
\end{align}
The second term in the inequality comes from the total variation distance that can make MDPs with large values be more common in $\bel$ than in $\Q$. The third term is similar, we can sample the policy from $\bel$ instead of $\Q$, with the added worst case penalty for the terms that differ.
	\end{proof}

\complexity*
    \begin{proof}
    \label{proof:complexity}
        The regret for an algorithm following the approximate posterior $Q$ is 
        \begin{align}
         \widetilde{\mathcal{O}}(\E_PR(\pi_Q)) &\le \widetilde{\mathcal{O}}(\sqrt{\Epindex}g(H,\mathcal{S},\mathcal{A})) + \sqrt{2}\Delta_{\max}\sum\limits_{k=1}^\Epindex \sqrt{\epslange} \\
         & \le \widetilde{\mathcal{O}}(\sqrt{\Epindex}g(H,\mathcal{S},\mathcal{A})) + \sqrt{2}\Delta_{\max}\sum\limits_{k=1}^\Epindex   \sqrt{C} \frac{g(H,\mathcal{S},\mathcal{A})}{\sqrt{t}\Delta_{\max}} \\
         & =\widetilde{\mathcal{O}}(\sqrt{\Epindex}g(H,\mathcal{S},\mathcal{A})) + \sqrt{2}
         g(H,\mathcal{S},\mathcal{A})
         \sqrt{C}\sum\limits_{k=1}^\Epindex    \frac{1}{\sqrt{t}}\\
         & =\widetilde{\mathcal{O}}(\sqrt{\Epindex}g(H,\mathcal{S},\mathcal{A})) 
        \end{align}
    \end{proof}

\samplecomplexity*

\begin{proof}
    \label{proof:samplecomplexity}
    This result can be obtained by directly applying the $\epslange$ obtained from \Cref{thm:approximate} into \Cref{thm:langevin_complexity} with $\gamma=n$.
\end{proof}

\regretcomplexity*
\begin{proof}
    \label{proof:applycomplex}
    Applying $\lsiep=\Omega(T)$ into  \Cref{thm:subgregret_lsi} we have that $\BRT=\sqrt{dH\totalT}$ with  $g(H,\mathcal{S}, \mathcal{A})=\Otilde(\sqrt{dH})$. Inserting $g(H,\mathcal{S}, \mathcal{A})$ into  \Cref{corr:sample_complexity} completes the proof.
\end{proof}

\chained*

    \begin{proof}
    \label{proof:chained}

    For notation we write $\bel(\M\mid \datatp) = \bel(\M\mid \datat, z_\epindex)$ such that $\bel(\M\mid \datatp) = \bel(\M\mid z_0, \ldots , z_{\epindex})$. Here we have that $z_\epindex=\{(s_i, a_i)\}_{i=H(\epindex-1)+1}^{H\epindex}$, the data for episode $\epindex$. Note that we can marginalize $\lik(z_\epindex \mid \datat, \M) = \lik(z_\epindex \mid \M)$ and $\E_\M \lik(z_\epindex \mid \datat, \M) = \lik(z_\epindex \mid \datat)$. 
    As a reminder, we have $\lik(z\mid \M)$ as the data likelihood for an episode, as such we have with Bayes rule $\bel(\M \mid \datat) = \frac{\bel(\M) \lik(\datat \mid \M)}{\lik(\datat)}$. Similarly, $\lik(x\mid \M)$ is the likelihood of a single step.

    \begin{align}
               &\KL{\hat{\bel}\left(\M|\datat\right)}{\bel(\M\mid \datatp)\mid z_\epindex}\\
        =& \int_{\M}\log\left(\frac{\hat{\bel}\left(\M|\datat\right)}{\bel(\M\mid \datat, z_\epindex)}\right)
        \hat{\bel}\left(\M|\datat\right)  d\M  \\
     = &\int_{\M}\log\left(\frac{\hat{\bel}\left(\M|\datat\right)}{\frac{\bel(\M\mid \datat)\lik(z_\epindex|\M)}{\lik(z_\epindex|\datat)}}\right)
        \hat{\bel}\left(\M|\datat\right) d\M  \\ 
 =&\int_{\M}\left(\log\left(\frac{\hat{\bel}\left(\M|\datat\right)}{\bel(\M\mid \datat)}\right) + \log\left(\frac{\lik(z_\epindex|\datat)}{\lik(z_\epindex|\M)}\right)\right)
        d\hat{\bel}\left(\M|\datat\right)  \\ 
     =& \int_{\M}\log\left(\frac{\hat{\bel}\left(\M|\datat\right)}{\bel(\M\mid \datat)}\right)         d\hat{\bel}\left(\M|\datat\right) \\
    & + \int_{\M} \log\left(\frac{\lik(z_\epindex|\datat)}{\lik(z_\epindex|\M)}\right)
        \hat{\bel}\left(\M|\datat\right) d\M \\ 
    =& \KL{\hat{\bel}\left(\M|\datat\right)}{\nu_\epindex} + \int_{\M} \log\left(\frac{\lik(z_\epindex|\datat)}{\lik(z_\epindex|\M)}\right)
       d \hat{\bel}\left(\M|\datat\right) \\
    \le &  \epslange +  \int_{\M} \log\left(\frac{\lik(z_\epindex|\datat)}{\lik(z_\epindex|\M)}\right)
        d \hat{\bel}\left(\M|\datat\right) %
    \end{align}
    The second equality comes from Bayes rule together with the marginalizations from above, the rest is separating of logarithms and identifying the desired KL-divergence. 
    
    This then gives that in expectation
    \begin{align}
        & \E_{z_\epindex}\KL{\hat{\bel}\left(\M|\datat\right)}{\bel(\M\mid \datatp)\mid z_\epindex} \\
     \le &\epslange +  \int_{z_\epindex}\int_{\M} \log\left(\frac{\lik(z_\epindex|\datat)}{\lik(z_\epindex|\M)}\right) 
           d \hat{\bel}\left(\M|\datat\right) {\lik(z_\epindex \mid \datat)} dz_\epindex \\
    = &\epslange +  \int_{\M} \int_{z_\epindex}\log\left(\frac{\lik(z_\epindex|\datat)}{\lik(z_\epindex|\M)}\right) 
        {\lik(z_\epindex \mid \datat)} dz_\epindex   d \hat{\bel}\left(\M|\datat\right)  \\
    = &\epslange +  \int_{\M} \int_{z_\epindex}\log\left(\frac{\lik(z_\epindex|\datat)}{\lik(z_\epindex|\M)}\right) 
        \frac{\lik(z_\epindex|\datat)}{\lik(z_\epindex|\M)} \lik(z_\epindex|\M)dz_\epindex  d \hat{\bel}\left(\M|\datat\right)   \\
    \label{eq:uselsi}
    \le & \epslange +  \int_{\M} 2/\lsiz \int_{z_\epindex}
    ||\nabla_z \sqrt{\frac{\lik(z_\epindex|\datat)}{\lik(z_\epindex|\M)}}||^2 \lik(z_\epindex|\M) dz_\epindex d \hat{\bel}\left(\M|\datat\right) \\
    = & \epslange +  \int_{\M} 2/\lsiz \int_{z_\epindex}
    ||\frac{\lik(z_\epindex|\M) \nabla_z \lik(z_\epindex|\datat) -\lik(z_\epindex|\datat)\nabla_z \lik(z_\epindex|\M)}{2 \sqrt{\frac{\lik(z_\epindex|\datat)}{\lik(z_\epindex|\M)}} \lik(z_\epindex|\M)^2}||^2 \lik(z_\epindex|\M) dz_\epindex  d \hat{\bel}\left(\M|\datat\right)    \\
    = & \epslange +  \int_{\M} 2/\lsiz \int_{z_\epindex}
    ||\frac{\lik(z_\epindex|\M) \nabla_z \lik(z_\epindex|\datat) -\lik(z_\epindex|\datat)\nabla_z \lik(z_\epindex|\M)}{2 \ \lik(z_\epindex|\datat) \lik(z_\epindex|\M)}\times \sqrt{\frac{\lik(z_\epindex|\datat)}{\lik(z_\epindex|\M)}}||^2 \lik(z_\epindex|\M) dz_\epindex  d \hat{\bel}\left(\M|\datat\right) \\
    = & \epslange +  \int_{\M} 2/\lsiz \int_{z_\epindex}
    ||\frac{\lik(z_\epindex|\M) \nabla_z \lik(z_\epindex|\datat) -\lik(z_\epindex|\datat)\nabla_z \lik(z_\epindex|\M)}{2 \ \lik(z_\epindex|\datat) \lik(z_\epindex|\M)}||^2 {\frac{\lik(z_\epindex|\datat)}{\lik(z_\epindex|\M)}} \lik(z_\epindex|\M) dz_\epindex  d \hat{\bel}\left(\M|\datat\right)    \\
    = & \epslange +  \int_{\M} \frac{1}{2\lsiz} \int_{z_\epindex}
    ||\nabla_z \log \lik(z_\epindex|\datat) - \nabla_z \log \lik(z_\epindex|\M)||^2 \lik(z_\epindex|\datat) dz_\epindex  d \hat{\bel}\left(\M|\datat\right)  \\
    \label{eq:gradlip}
     \leq & \epslange +  \frac{{L_z}^2}{2\lsiz} \int_{\M} 
    ||\E_{\bel(\M|\datat)}[\M] - \M ||^2 d \hat{\bel}\left(\M|\datat\right) \\
    \leq & \epslange + \frac{{L_z}^2}{2\lsiz} \int_{\M} 
    ||\E_{\bel(\M|\datat)}[\M] - \E_{\hat{\bel}\left(\M|\datat\right)}[\M] ||^2 d \hat{\bel}\left(\M|\datat\right) +  \frac{{L_z}^2}{2\lsiz} \int_{\M} 
    ||\E_{\hat{\bel}\left(\M|\datat\right)}[\M] - \M ||^2 d \hat{\bel}\left(\M|\datat\right)  
    \label{eq:addsubtract}\\
    \leq & \epslange + \frac{{L_z}^2}{\lsiz} \epslange^2 \mathrm{Var}(\bel\left(\M|\datat\right)) +  \frac{{L_z}^2}{2\lsiz} \mathrm{Var}(\hat{\bel}\left(\M|\datat\right)) %
    \label{eq:transportchain}
\end{align}
The inequality in \Cref{eq:uselsi} comes from the log-Sobolev inequality property of $\lik(z_\epindex|\M)$. \Cref{eq:gradlip} which comes from the Lipschitz property of the gradient, \Cref{eq:addsubtract} is an addition and subtraction with the norm triangle inequality and \Cref{eq:transportchain} comes from the definition of variance together and the transportation result of \Cref{eq:transport}.  The rest of the steps are algebraic manipulations.
\end{proof}

\section{LaPSRL for Different Families of Distributions}
In this section we present the details on LaPSRL for different families of distributions which can be found in \Cref{tab:regret}. 
The results for the Bayesian regret follow from combining the scaling of the log-Sobolev constant with \Cref{thm:subgregret_lsi}. 

\textbf{Mixture Distributions.} There has been multiple works on LSI constants for mixtures of log-Sobolev distributions~\citep{koehler2024sampling,CHEN2021109236,schlichting2019poincare}. Generally, LSI of a mixture depends on LSI constants of the components and the distance between them. %
\begin{restatable}[Informally from Theorem 2 \citep{koehler2024sampling}]{theorem}{mixturelsi}
\label{theorem:mixturelsi}
    For k-mixture components  $\mu=\sum_{i=1}^k p_i\mu_i, \sum_{i=1}^k p_i=1$, where there is some overlap $\delta$ between components, has $\lsi_{\text{Mixture}}\ge\frac{\delta \min p_i \min \lsi_i}{4k(1-\log(\min p_i))}$.
\end{restatable}
  The overlap factor $\delta$ relates to integral over the minimum of the paired components~\citep{koehler2024sampling}. If the components are posterior distributions, $\delta \rightarrow 1$ as individual posteriors observe more data and converge.

\lsigeneral*
 \begin{proof}
 \label{proof:lsigeneral}
    We can write the product of log-concave distributions  $\bel(\theta\mid \datat)=\prior(\theta)\frac{\prod_{i=1}^n \lik_i(\theta)}{Z}$ where  $\lik_i(\theta)$ is shorthand for $\lik(x_i \mid \theta)$ with $x_i$ the datapoint at time $i$. 
    Since products preserve log-concavity, we can use \Cref{theorem:bakry-emery}. 
    From Weyl's inequality, we have that the smallest eigenvalue a sum of two Hermitian is larger than the sum of the smallest eigenvalues of the two matrices. Since the Hessian is a Hermitian matrix, putting this into \Cref{theorem:bakry-emery} this gives that $\lsiep\ge \lsi_{\prior(\theta)} +\sum_{i=1}^n \lsi_i\ge \lsi_{\prior(\theta)} + n \min_i \lsi_i$. Similarly, applying Weyl's inequality for the largest eigenvalue, we get that the largest eigenvalue of $-\nabla^2\log(\bel(\theta\mid \datat))$ is upper bounded by the sum of maximal eigenvalues which gives an upper bound of $O(n)$ for $\lsiep$ since the smallest eigenvalue must be smaller than the largest.

    Similarly, for mixtures of log-concave distributions we have from \Cref{theorem:mixturelsi} that $\lsi_{\text{Mixture}} = \Omega\left(\frac{ \min \lsi_i \min p_i}{4k(1-\log(\min p_i)}\right)$. Setting $\min_i \lsi_i=\Theta(n)$ completes the proof.
    \end{proof}

\boundedlik*
\begin{proof}
    This result follows directly from \Cref{theorem:holley} where the log-likelihood ratio gives the difference between the maximum and minimum perturbations.
\end{proof}

\begin{table*}[h!]
    \centering
    \caption{Overview of log-Sobolev constants and Bayes regrets of LaPSRL for different families of distributions.}\label{tab:regret}
    \begin{tabular}{c c c}
    \toprule
    Posterior & log-Sobolev constant& LaPSRL Bayesian regret \\
    \midrule
    Gaussian & $\frac{1}{\sigma_0^2} + \frac{n}{\sigma^2}$&$\Otilde\left(\left(L_{\Bar{R}}  + \E[L_\M] L_{\Bar{\transition}}\right)\sqrt{T\sigma^2}\right)$ \\
    Log-concave & $\Theta(n)$ &  $\Otilde\left(\left(L_{\Bar{R}}  + \E[L_\M] L_{\Bar{\transition}}\right)\sqrt{T}\right)$ \\
    Mixture of Log-concave  & $\Omega\left(\frac{\delta \min p_i  \min \lsi_i}{4k(1-\log(\min p_i))}\right)$&$\Otilde\left(\left(L_{\Bar{R}}  + \E[L_\M] L_{\Bar{\transition}}\right)\sqrt{\frac{4kT}{\min p_i}}\right)$\\
    \bottomrule
\end{tabular}
\end{table*}

    \section{Experimental details}
    In this section we go into more neccesary details for the experiments. We also re-include the plots from the main paper in \Cref{fig:app_plots} for increased visibility.
    	\subsection{Gaussian Multi-armed Bandits}
 We use LaPSRL on a Gaussian multi-armed bandit task with two arms. The arms generate rewards as $\mathcal{N}(0,0.25)$, $N(0.1, 0.25)$. To preserve computations, we use a batched approach such that the same action is played for 20 steps. As a baseline, we compare with the performance of  PSRL from the true posterior. Both LaPSRL and Thompson sampling use a $\mathcal{N}(0,1)$ prior for the mean of each arm. Additionally, we compare with a LaPSRL algorithm that has a bimodal $1/2 \mathcal{N}(0,1/4) + 1/2 \mathcal{N}(1,1)$ prior over the arms.  In this experiment we use chained sampling in the Langevin algorithm such that we initialize with the previous step.  The results can be seen in \Cref{fig:app_bandit}. %

	\subsection{Continuous MDPs}
We evaluate LaPSRL on two continuous environments, Cartpole and Reacher. 
The Cartpole environment is a modified version of Cartpole environment~\citep{6313077} to have continuous actions, with states $s\in \R^4$ and a continuous action in [-1,1]. The goal is to control a cart such that the attached pole stays upright. 
    	We use a Linear Quadratic Regulator model, where LaPSRL samples from a distribution over the $A$ and $B$ matrixes with a  $\mathcal{N}(0,1)$  prior over the values. The policy can then be obtained through the Riccati equation.  Instead of calculating the log-Sobolev constant for the posterior distribution, we just evaluate for a variety of $\lsi \in \left \{100n, 1000n, 10000n\right\}$. To simplify the parameter search, we set the $L$ parameter to $\lsi/n$. Instead of estimating $\log \frac{2 \KL{\rho_0}{\lik(\M\mid \datat)}}{\epslange}$, we upper bound this with $n$. In each sampling step, we start with an initial sample from $\mathcal{N}(0,1)$. In the learning, we assume that the error is Gaussian standard deviation of 0.5. 
 While Cartpole is not a linear MDP, but it is a good approximation and serves to show that LaPSRL can work even when the true model is not part of the posterior support. 
 As a baseline we compare with an exact PSRL algorithm which samples from Bayesian linear regression priors~\citep{minka2000bayesian}. %
 Finally, we use a variant  of LaPSRL with a multimodal prior over the $A$ and $B$ matrixes with a $1/2  \mathcal{N}(0,1) + 1/2  \mathcal{N}(1,0.25)$ to demonstrate that it also works well for multimodal priors that are not log-concave. The results from this experiment can be found in \Cref{fig:app_cartpole} where we plot what fraction of the 50 runs have solved the task (i.e. taking 200 steps without failing). Here we see that all versions successfully handle the task, even faster than the PSRL baseline. We can note that it takes longer for the experiments with larger $\lsi$ values to converge. 

The Reacher environment is a standard environment from the Gymnasium environments~\citep{towers2024gymnasium} relying on MuJoCo~\citep{todorov2012mujoco} physics simulations. In the Reacher environment the agent controls an arm and has to stay close to a randomly placed target. Here we use a neural network model, this means that this is not necessarily log-Sobolev, but we wish to show that this approach still is useful. The neural network parameters were sampled from  $\mathcal{N}(0,0.1)$. In the learning, we assume that the error is Gaussian standard deviation of 0.5. Here we let LaPSRL know the reward function, and the agent only models the movement of the robot arm. In the state of Reacher there are some states that are constant (the postion of the target) and some that are functions of other states (distance to the target), LaPSRL does not model these values. Policy rollouts are then done in the model using a cross-entropy method (iCEM) ~\citep{PinneriEtAl2020:iCEM} to find a policy to use in the environment. For the iCEM algorithm we use 8 iterations, 48 trajectories in each and keep 5 elite samples. Since the environment is deterministic, we do not use any noise in the rollouts. We model the transitions using a neural network with two hidden layers with 128 neurons. In this Reacher experiment, we use chained sampling, but also fix the amount of steps in each sampling to 150000.
We vary the LSI constant $\lsi \in \left \{10000n, 100000n\right\}$ and as in Cartpole experiment set the $L$ parameter to $\lsi/n$. 
Here we benchmark with the TD3~\citep{fujimoto2018addressing} and PPO~\citep{schulman2017proximalpolicyoptimizationalgorithms}algorithms. As PPO does not use a replay buffer and the CleanRL implementation has a large batch size, it is significantly less data efficient, and we present the average performance after 7000 episodes.
We can see that
\Cref{fig:app_reacher} that LaPSRL can learn the environment very quickly and reaches a performance that takes TD3 reaches eventually, but takes significantly longer.

\begin{figure}[t!]
\label{fig:app_plots}
        \centering
        \begin{subfigure}[t]{0.45\linewidth}
        \centering
        \includegraphics[width=\linewidth]{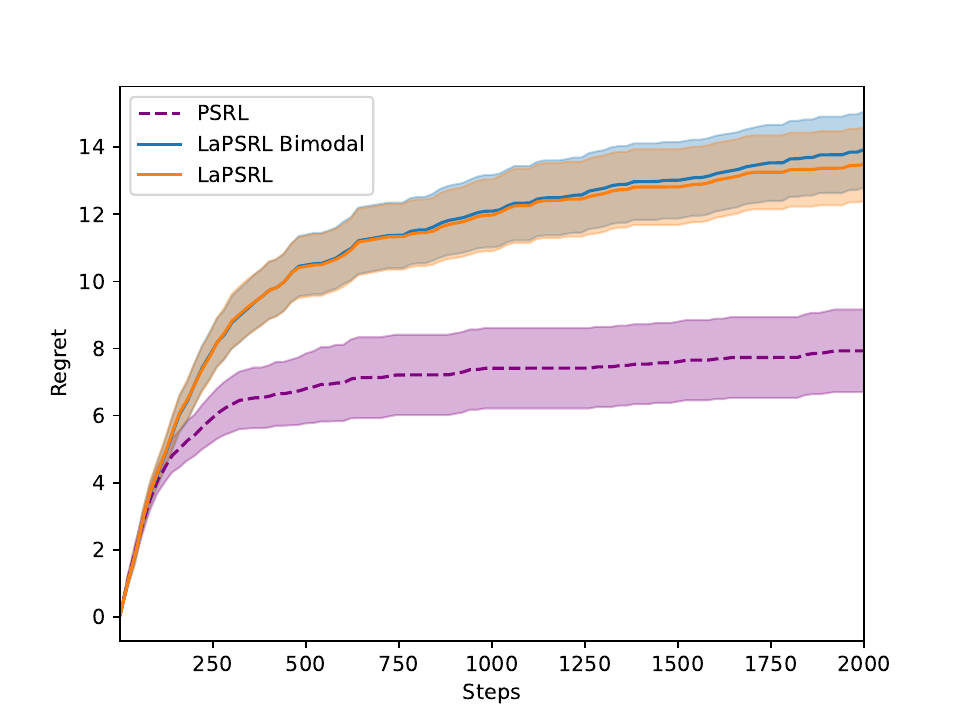}
        \label{fig:app_bandit}
        \caption{Gaussian Batched Bandit}
    \end{subfigure}   
    \begin{subfigure}[t]{0.45\linewidth}
        \centering
\includegraphics[width=\linewidth]{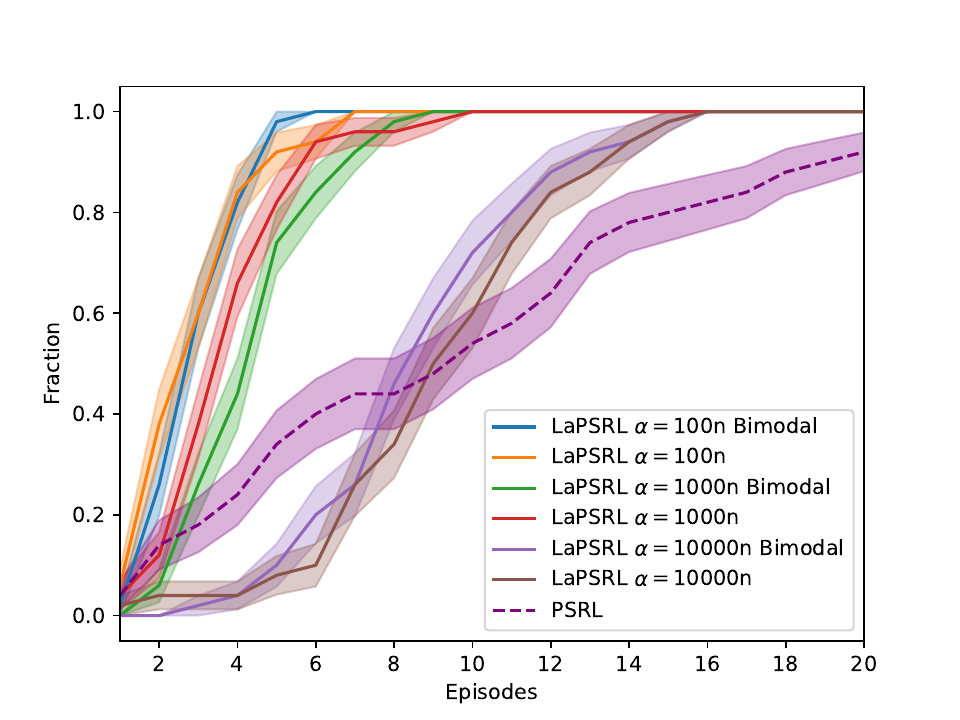}
        \label{fig:app_cartpole}
        \caption{Cartpole}
    \end{subfigure}\\
        \begin{subfigure}[t]{0.45\linewidth}
        \centering
\includegraphics[width=\linewidth]{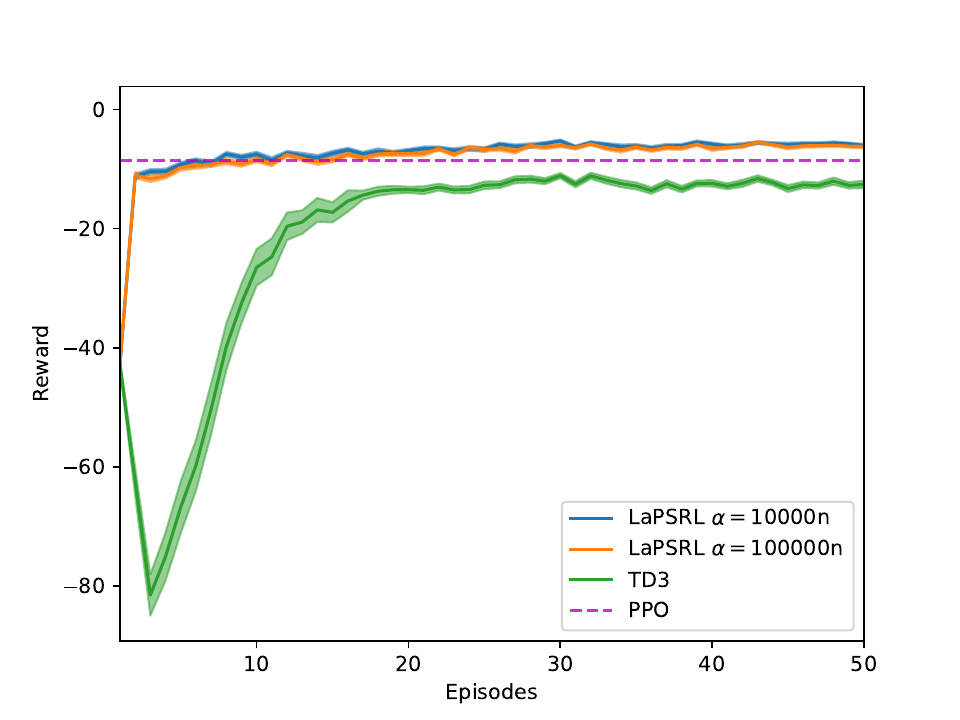}
        \label{fig:app_reacher}
        \caption{Reacher}
    \end{subfigure}
\caption{We compare LaPSRL versus baselines. In the bandit and Cartpole experiments we benchmark with PSRL, in Reacher with TD3. In a) we compare the expected regret for a Gaussian bandit algorithm. In b)  we compare how many episodes it takes to solve a Cartpole task. In c) we study the average regret per episode in the Reacher environement. In all environments, we average over 50 independent runs with the standard error highlighted around the average. }
\end{figure}
    
\subsection{Computational Notes}
        The experiments have been done primarily in Jax~\citep{jax2018github}. The reacher environment is from the Gymnasium environments~\citep{towers2024gymnasium} relying on MuJoCo~\citep{todorov2012mujoco} physics simulations. The Cartpole is a modified version of Cartpole environment~\citep{6313077} to have continuous actions. For the experiments with  TD3 \citep{fujimoto2018addressing} we used the implementation from CleanRL \citep{huang2022cleanrl}.
        The experiments were run on an internal cluster to be able to run many experiments at once, but they will also run on a regular desktop. See the supplementary code for additional details.

\end{document}%